\documentclass[a4paper,notitlepage]{article}

\usepackage[a4paper,top=2cm,bottom=2cm,left=2.5cm,right=2.5cm,marginparwidth=1.75cm]{geometry}

\usepackage{authblk}

\usepackage{hyperref}
\hypersetup{
  colorlinks   = true, 
  linkcolor    = black!50!brown, 
  citecolor   = black!50!brown, 
}

\RequirePackage{algorithm}
\RequirePackage{algorithmic}

\usepackage{bm}
\usepackage{graphicx,xcolor}

\usepackage{mathtools,tikz}
\usepackage{hhline}
\usepackage{multirow}
\usepackage{enumitem}  
\usepackage{caption}
\usepackage{subcaption}

\usepackage{appendix}

\usepackage{cite}
\usepackage{amsmath,amssymb,amsfonts,amsthm,xspace}
\theoremstyle{definition}

\newtheorem{remark}{Remark}
\newtheorem{theorem}{Theorem}

\newtheorem*{defn*}{Definition}
\newtheorem*{lemma*}{Lemma}

\newtheorem{corollary}[theorem]{Corollary}

\newcommand{\h}[2]{\ensuremath{\mathsf{d}_{\mathsf{H}}(#1, #2)}} 
 
\newcommand{\ipr}[2]{\ensuremath{\langle #1, #2\rangle} }

\newcommand{\norm}[1]{\left\lVert#1\right\rVert_2}
\newcommand{\normi}[1]{\left\lVert#1\right\rVert}

\foreach \x in {A,...,Z}{%
\expandafter\xdef\csname vec\x \endcsname{\noexpand\ensuremath{\noexpand\mathbf{\x}}}
}

\foreach \x in {A,...,Z}{%
\expandafter\xdef\csname c\x \endcsname{\noexpand\ensuremath{\noexpand\mathcal{\x}}}
}

\foreach \x in {A,...,Z}{%
\expandafter\xdef\csname bb\x \endcsname{\noexpand\ensuremath{\noexpand\mathbb{\x}}}
}

\DeclareMathOperator*{\argmax}{arg\,max}

\newcommand\reals{{\mathbb R}}
\newcommand{\wh}[1]{\ensuremath{{\left|#1\right|}_{\mathsf{H}}}} 

\newcommand{\inb}[1]{\left\{#1\right\}}
\newcommand{\inp}[1]{\left(#1\right)}
\newcommand{\insq}[1]{\left[#1\right]}
\newcommand{\inl}[1]{\left|#1\right|}
\newcommand{\bcs}{\ensuremath{\mathsf{1bCSbinary}}}
\newcommand{\logreg}{\ensuremath{\mathsf{LogisticRegression}}}
\newcommand{\spl}{\ensuremath{\mathsf{SparseLinearReg}}}

\newcommand{\bx}{\mathbf{x}}
\newcommand{\ba}{\mathbf{a}}
\newcommand{\by}{\mathbf{y}}
\newcommand{\bz}{\mathbf{z}}

\newcommand{\bw}{\mathbf{w}}
\newcommand{\bb}{\mathbf{b}}
\newcommand{\be}{\mathbf{e}}
\newcommand{\sign}[1]{\mathsf{sign}(#1)}

\title{Exact Recovery of Sparse Binary Vectors from \\ Generalized Linear Measurements}
\author[1]{Arya Mazumdar}
\author[1]{Neha Sangwan}

\affil[1]{\small Halicioglu Data Science Institute, University of California San Diego, La Jolla, United States}

\begin{document}
\maketitle

\begin{abstract}
 We consider the problem of {\em exact} recovery of a $k$-sparse binary vector  from generalized linear measurements (such as {\em logistic regression}). We analyze the {\em linear estimation} algorithm (Plan, Vershynin, Yudovina, 2017), and also show information theoretic lower bounds on the number of required measurements. As a consequence of our results, for  noisy one bit quantized linear measurements (\bcs), we obtain a sample complexity of $O((k+\sigma^2)\log{n})$, where $\sigma^2$ is the noise variance. This is shown to be optimal  due to the information theoretic lower bound. 
{We also obtain tight sample complexity characterization for logistic regression.}

  Since \bcs\ is a strictly harder problem than noisy linear measurements (\spl) because of added quantization, the same sample complexity is achievable for \spl. 
  While this sample complexity  can be obtained via the popular lasso algorithm, linear estimation is  computationally more efficient. 
  Our lower bound  holds for any set of measurements for \spl\, (similar bound was known for Gaussian measurement matrices) and is closely matched by the maximum-likelihood upper bound. 
  For \spl,  it was conjectured in Gamarnik and Zadik, 2017 that there is a statistical-computational gap and the number of measurements should be at least $(2k+\sigma^2)\log{n}$ for efficient algorithms to exist. It is worth noting that our results imply that there is 
   no such statistical-computational gap for \bcs\ and logistic regression.
\end{abstract}

\section{Introduction}\label{sec:intro}
Sparse linear regression and compressed sensing have been a topic of intense research in statistics and signal processing for the past few decades~\cite{candes2006robust,donoho2006compressed,tibshirani1996regression}.
The problem of \textbf{binary} sparse linear regression (\spl) considers linear measurements of an unknown binary vector, corrupted by additive Gaussian noise. Focusing on binary signals, this particular problem has recently been studied in ~\cite{david2017high,gamarnik2017sparse,gamarnik2022sparse,pmlr-v99-reeves19a}, mainly motivated by the question of {\em support recovery} of sparse signals~\cite{wainwright2009sharp}. Formally, for an unknown $k$-sparse signal $\bx\in \inb{0,1}^n$, a sensing matrix $\vecA\in \reals^{m\times n}$ and a noise vector $\bz = (z_1, \ldots, z_m)$ where $z_i$s are iid $\cN(0, \sigma^2)$ for some variance $\sigma^2$, we observe $\by$ given by 
\begin{align}
    \by = \vecA \bx + \bz.\label{eq:spl}
\end{align} 
Our goal is to design the (possibly random) sensing matrix $\vecA$ with a power constraint, i.e., 
\begin{align}
  \bbE[(\vecA_i^T \bx)^2] \leq k, i =1, \dots, m,\label{eq:power_constraint}   
\end{align}
where {$\vecA_i$ denotes $i^{th}$ row of the matrix $\vecA$ and} the expectation is over the possible randomness in $\vecA$, and a decoding algorithm $\phi$ such that
\begin{align}
\max_{\bx\in\inb{0,1}^n, \wh{\bx} = k}\bbP\inp{\phi(\vecA, \by) \neq \bx}\rightarrow 0 \text{ as } n\rightarrow \infty.\label{eq:pe_spl}
\end{align} 
Here, $\wh{\bx}$ denotes the Hamming weight of $\bx\in \inb{0,1}^n$. The probability is computed over the randomness of the sensing matrix and the (randomized) algorithm. 

The problem of one bit quantized linear measurements (also known as one bit compressed sensing (\bcs)) is similar, except that  the output vector $\by$ is the sign of $\vecA\bx+\bz$ instead of the entire vector $\vecA\bx+\bz$~\cite{DBLP:conf/ciss/BoufounosB08}. That is, we observe
\begin{align}
\by  = \sign{\vecA\bx+\bz}.\label{eq:bcs}
\end{align}Here, $\by = (y_1, \ldots, y_m)$ is defined as $y_i = \sign{\vecA_i^T \bx+z_i}$, $i\in [1:m]$ where $\sign{a} = 1$ if $a\geq 0$ and $\sign{a} = -1$ otherwise. An algorithm $\phi'$ for \bcs\ takes input  $\by$ and $\vecA$. Again, we require that 
\footnote{The probability
 of error measured by \eqref{eq:pe_1bcs} 
 corresponds to the {\em`for each'} criterion in the one bit compressed sending literature. The {\em`for all'} criterion which requires that the same sensing matrix works for all unknown signals corresponds to showing $\bbP\inp{\exists \bx \text{ such that }\phi'(\vecA, \by) \neq \bx}\rightarrow 0$ as $n\rightarrow \infty$.}
\begin{align}
\max_{\bx\in\inb{0,1}^n, \wh{\bx} = k}\bbP\inp{\phi'(\vecA, \by) \neq \bx}\rightarrow 0 \text{ as } n\rightarrow \infty.\label{eq:pe_1bcs}
\end{align} 
 Usually in 1-bit compressed sensing, the Gaussian noise before quantization is not present. Our  formulation  can be considered as a {\em sparse} ``probit model''~\cite{mccullagh2019generalized}.

More generally, we define the problem of generalized linear measurements (GLMs) , e.g., \cite{kakade2011efficient,vershyninPlan} where we assume that the observation $\by = \inp{y_1, \ldots,y_m}$ is related to the sparse binary input vector $\bx$ using an ``inverse link'' function $g$ such that for each $i\in [m]$, 
\begin{align}
    \bbE\insq{y_i|\vecA_i} = g\inp{\vecA_i^T\bx}.\label{eq:glm}
\end{align} 
That is, the expected value of the output $y_i$ is linked to $\vecA_i$ only through $\vecA_i^T\bx$. For example, for \spl\ $$\bbE\insq{y_i|\vecA_i} = \vecA_i^T\bx,$$ for \bcs\ $$\bbE\insq{y_i|\vecA_i} = 1-2\Phi\inp{\frac{-A_i^T\bx}{\sigma}}$$ where $\Phi$ is the Gaussian cumulative distribution function.

In the logistic regression model (\logreg), we observe a binary output $y_i\in \inb{-1, 1}$ for each measurement $i\in [m]$. The probability that $y_i$ takes value 1 is given by 
\begin{align*}
\bbP\inp{y = 1} = \frac{1}{1+e^{-\beta \ba^T\bx}}.
\end{align*} for parameter $\beta>0$. The parameter $\beta$ controls the level of noise. When $\beta \rightarrow \infty$, the model approaches  noiseless one bit compressed sensing. As $\beta$ decreases, the output becomes more noisy. When $\beta = 0$, the output is uniformly distributed on $\inb{-1,1}$ and is independent of $\bx$.
In this model,  $$\bbE\insq{y_i|\vecA_i} = \tanh{\frac{\beta\vecA_i^T\bx}2}.$$ 

\paragraph{Our contributions.} In this paper, our contributions are the following:
\begin{itemize}
    \item We analyze the linear estimation+projection algorithm \cite{vershyninPlan}   for generalized linear measurements of sparse binary inputs (Theorem~\ref{thm:alg_general}). We also provide an information theoretic lower bound (Theorem~\ref{thm: lower_bdglm}).
    \item As corollaries, we obtain tight sample complexity characterization for noisy one bit compressed sensing   (Corollary~\ref{thm:alg_bcs} and Corollary~\ref{thm: lower_bd_bcs}) and logistic regression (Corollary~\ref{thm:alg_logreg} and Corollary~\ref{thm: lower_bd_log_reg}). 
    \item The algorithm can be used for \spl\ either directly (Corollary~\ref{thm:alg_spl}) or by first quantizing the received signal to its sign value and then using the algorithm for \bcs. The sample complexity is the same for both these cases. This shows that in the regime where the number of measurements are at least $C(k+\sigma^2)\log{n}$ for some constant $C$, keeping only the sign information is sufficient for \spl. 
    \item We provide ``almost'' matching information theoretic lower (Corollary~\ref{thm: spl_lower_bd_1}) and upper bounds (Theorem~\ref{thm:upper_bd_mle}) for exact recovery in \spl. If the measurements are Gaussian, we get slightly better lower bounds (Theorem~\ref{thm:lower_bd_spl}).
\end{itemize}

\subsection{Discussion of results and related works}
\paragraph{Intuitions on lower bounds.} Observe that \bcs\ is a strictly more difficult problem than \spl\ in the sense that any algorithm that works for \bcs\ can be used for \spl\ by using only the sign information. Thus, the sample complexity of \bcs\ is at least as much as \spl, the latter can be much smaller in some cases. From an information theoretic viewpoint, a randomly chosen $k$-sparse vector $\bx$ has entropy $\log{n \choose k}\approx k\log\frac{n}{k}$. Since each $y_i$ can give at most one bit of information, we need at least $k\log\frac{n}{k}$ measurements for \bcs\ (See Corollary~\ref{thm: lower_bd_bcs} for the exact lower bound) to learn $\bx$. For \spl\ on the other hand, the output has infinite precision. In fact, we can show that in the absence of noise, only one sample is sufficient to recover the unknown signal (see Remark~\ref{remark:noNoise}). \spl\ can be viewed as a coding problem for a Guassian channel, where $\bx$ is the message and $\vecA\bx$ is its corresponding codeword. Thus, from the converse for Gaussian channel (see [Theorem 9.1.1] \cite{thomas2006elements}), we need at least $\frac{k\log\inp{n/k}}{C}$ samples for exact recovery. Here $C$ is the capacity of the Gaussian channel, which depends on SNR (a function of $\vecA\bx$ and $\sigma^2$). 
Given the power constraint of Eq.~\eqref{eq:power_constraint} (which is satisfied when entries of $\vecA$ are chosen iid $\cN(0,1)$),  the capacity $C$ is $\frac{1}{2}\log\inp{1+\frac{k}{\sigma^2}}$, thereby showing that the lower bound for \spl\ can be much smaller.

\paragraph{Binary sparse linear regression.} The problem of binary sparse linear regression was introduced in \cite{david2017high, gamarnik2022sparse} and was further studied in \cite{pmlr-v99-reeves19a}. An {\em``all or nothing''} phenomenon was shown in \cite{pmlr-v99-reeves19a} for {\em approximate recovery} of binary vectors at the critical sample complexity of $m^* \triangleq \frac{2k\log{n/k}}{\log\inp{1+\frac{k}{\sigma^2}}}$, showing that approximate recovery is possible if and only if $m\geq m^*$. It was additionally conjectured in \cite{david2017high} that no efficient algorithms exist in the regime $m^*\leq m \leq m_{\mathsf{alg}} \triangleq (2k+\sigma^2)\log{n}$. When $m\geq m_{\mathsf{alg}}$, various algorithms like Lasso~\cite{wainwright2009sharp}, Orthogonal Matching Pursuit (OMP)~\cite{tropp2007signal} and \cite{ndaoud2020optimal} can recover the sparse vector. It has also been shown in~\cite{gamarnik2017sparse} that lasso fails to recover unknown vector $\bx$ when $m\leq c\, m_{\mathsf{alg}}$ for some small constant $c$. Outside this regime, a local search algorithm was proposed~\cite{gamarnik2017sparse}, which starts with a guess of $\bx$ and iteratively updates it.

In \cite{pmlr-v99-reeves19a}, the information theoretic lower bound of $m^*$ is shown for the case when each entry of the sensing matrix is chosen iid $\cN(0,1)$. 
We consider the exact recovery guarantee for the problem and show that $m \geq m^*$ samples are necessary even when the sensing matrix is not Gaussian (Theorem~\ref{thm: spl_lower_bd_1})\footnote{Our lower bounds hold for a weaker average probability of error recovery criteria, instead of the maximum probability of Eq.~\eqref{eq:pe_spl}, hence are more potent.}. We show an almost matching upper bound based on the Maximum Likelihood Estimator (MLE) using a random Gaussian sensing matrix (Theorem~\ref{thm:upper_bd_mle} and Theorem~\ref{thm:lower_bd_spl}). This is along the lines of the MLE analysis in \cite{pmlr-v99-reeves19a}, which was done for approximate recovery (our sample complexity for exact recovery turns out to be slightly different).   
\begin{remark}\label{remark:noNoise}
It was observed in \cite{david2017high} that in the no-noise regime ($\sigma^2 = 0$), one measurement is sufficient to recover the underlying vector by brute force. However, it is conjectured that there is no efficient algorithm if $m\leq 2k\log{n}$. The results in \cite{david2017high} were shown only when the entries of the sensing matrix are chosen iid $\cN(0,1)$ (i.e. Gaussian design). For an arbitrary sensing matrix, an efficient way to recover $\bx$ using only one measurement is by using $\vecA = \frac{1}{2^n}[1, 2, 2^2, \ldots,2^{n-1}]$. Note that $2^n \times \by$ in this case is the value of unknown signal in the decimal system (base 10). It can be converted to binary in $O(n)$ time. This suggests that for specific non-random constructions, there may be efficient algorithms in the conjectured hardness regime.
\end{remark}

\paragraph{Binary one-bit compressed sensing.} The problem of  one bit compressed sensing has been well studied~e.g.~\cite{DBLP:conf/ciss/BoufounosB08,jacques2013robust} including greedy algorithms (e.g.~\cite{liu2016one}) and noisy test outcomes (e.g.~\cite{matsumoto2024robust}), and the problem of recovering binary vectors has also been studied in~\cite{acharya2017improved,mazumdar2022support}. However, these works do not consider the Gaussian noise prior to quantization. The best known upper bound ($O(k/\epsilon)$ from~\cite{matsumoto2022binary}) when specialized to exact recovery for binary sparse vectors requires $O(k^{3/2})$ (by choosing $\epsilon = 1/\sqrt{k}$). On the other hand, our bound is $O(k\log{n})$. This discrepancy is because the previous models are studied for the ``for all'' model which is a harder problem than our present ``for each'' model. The results in \cite{vershyninPlan}, on the other hand, are for the ``for each'' model, though their analysis is not optimal for binary vectors (see Appendix~\ref{sec:comparison_PV}). The problem of noisy one bit compressed sensing (\bcs) introduced here is motivated by the probit model (e.g. see a modern treatments of the non-sparse probit model~\cite{kuchelmeister2024finite}). Here we provide an information theoretic lower bound of $m\geq \inp{k+\sigma^2}\log\inp{n/k}$ and show that the aforementioned efficient algorithm  (Algorithm~\ref{alg:1}) works with the same $m =O((k+\sigma^2)\log{n})$ samples and has a computational complexity of $O((k+\sigma^2)n\log{n})$. We also provide optimal sample complexity characterization for learning binary sparse vectors under the logistic regression model, which was previously studied for learning real vectors \cite{hsu2024sample,vershyninPlan}.

\paragraph{Algorithm for binary vectors.} We consider a simple algorithm which is equivalent to the ``average algorithm''~\cite{servedio1999pac} or ``linear estimator''~\cite{vershyninPlan}, followed by a selection of the `top-k' coordinates. Regarding the intuition behind the algorithm, we observe that for an unknown signal $\bx$, the output $\by$ and $\vecA_{\cS_{\bx}}$, the restriction of the sensing matrix to columns where $\bx$ is 1,  are correlated whereas $\by$ and $\vecA_{[1:n]\setminus \cS_{\bx}}$ are uncorrelated. Here, $\vecA_{[1:n]\setminus \cS_{\bx}}$ denotes the restriction of the sensing matrix to columns where $\bx$ is 0.  Thus, we compute the inner product between $\by$ and each column of the sensing matrix as a proxy for correlation between the output and the corresponding column. The output of the algorithm is the top $k$-most correlated columns (See Algorithm~\ref{alg:1} for details.). One can also think of this as a ``one-shot'' version of the popular OMP algorithm. This algorithm requires $O((k+\sigma^2)\log{n})$ samples for \bcs\ and \spl\, and $O((k+1/\beta^2)\log{n})$ for \logreg. It has a computation complexity of $O((k+\sigma^2)n\log{n})$. Most of the previous algorithms, including the one in \cite{vershyninPlan}, were given for the case when the unknown signal is not necessarily binary. It should be noted that the black-box application of the result of \cite{vershyninPlan}  specialized to binary inputs will not recover the optimal sample complexity. See Appendix~\ref{sec:comparison_PV} where we show that the results in \cite{vershyninPlan} imply a sample complexity of $O(k^2\log\inp{2n/k})$. We provide a simple yet  optimal analysis of the sample complexity in our special case of  sparse binary signals. 

We would like to emphasize that the sample complexity of Algorithm~\ref{alg:1} for both \spl\ and \bcs\ is the same ($O((k+\sigma^2)\log{n})$). 
This implies that for \spl, when $m$ is outside the conjectured hardness regime, we do not need the amplitude of $\by$, only the sign information is sufficient to recover the unknown signal. 


\paragraph{Notation.}We will use boldfaced uppercase letters like $\vecA$ for matrices and lowercase letters such as $\bx$ for vectors. The entry of the matrix at $i^{\text{th}}$ row and $j^{\text{th}}$ column is denoted by $A_{i,j}$. Similarly, the $i^{\text{th}}$ entry of a vector $\bx$ is denotes by $x_i$.
For any binary vector $\bx = \inp{x_1, \ldots, x_n}$, we denote the set of indices $i$ where $x_i = 1$ by $\cS_{\bx}\subseteq [1:n]$ and we use $\vecA_{\cS_{\bx}}$ to denote the restriction of $\vecA$ to the columns where $\bx$ is 1. We use $\vecA_i$ to  denote $i^{\text{th}}$ row or $i^{\text{th}}$ column, depending on context. The correct notation is made clear where it is used. We denote the binary entropy function by $h_2(\cdot)$.

\paragraph{Organization.} We present the algorithm  and upper bounds in Section~\ref{sec:alg}. The information theoretic lower bounds are presented in Section~\ref{sec:sample_compexity}.  In Section~\ref{sec:tighter_bounds_spl}, we present an upper bound for \spl\ based on the maximum likelihood estimator. We also provide a lower bound in this section, which closely matches the upper bound. This lower bounds  does not follow as a corollary to the general lower bound theorem for GLMs (Theorem~\ref{thm: lower_bdglm}). It requires a separate analysis based on a conditional version of Fano's inequality. We provide proofs of the upper and lower bound for GLMs (Theorems~\ref{thm:alg_general} and \ref{thm: lower_bdglm}) in Section~\ref{sec:proofs}. Remaining proofs are delegated to Appendix~\ref{appendix:proofs}. We provide detailed comparison of our results with \cite{vershyninPlan} in Appendix~\ref{sec:comparison_PV}. We conclude with a discussion on open problems in Section~\ref{sec:conclusion}.

\section{Main results}

\subsection{Algorithm}\label{sec:alg}
We analyze the simple linear estimation based algorithm from \cite{vershyninPlan} for generalized linear measurements, specializing it for binary vectors.
The algorithm (Algorithm~\ref{alg:1}) takes the sensing matrix $\vecA$ and the output vector $\by$ as the inputs. 
For each column $\vecA_i,\, i\in [1:n]$ of the sensing matrix, the algorithm computes $l_i = \ipr{\by}{\vecA_i} = \sum_{j = 1}^{m}y_jA_{j,i}$ where $A_{j,i}$ is the entry at $j^{\text{th}}$ row and $i^{\text{th}}$ column.

The vector $\mathbf{l} = \inp{l_1, \ldots, l_n}$ is then sorted in decreasing order. The output of the algorithm is a set containing the indices of the top-$k$ elements of the sorted vector. That is, if the sorted vector is $\inp{l_{\alpha_1}, l_{\alpha_2}, \ldots, l_{\alpha_n}}$ where $l_{\alpha_i}\geq l_{\alpha_j}$ for $i\leq j$, then the output of the algorithm is  $\cS = \inb{\alpha_1, \ldots, \alpha_k}$.

\begin{algorithm}[tbh!]
   \caption{Top-$k$ correlated indices}
   \label{alg:1}
\begin{algorithmic}
   \STATE {\bfseries Input:} Sensing matrix $\vecA\in \bbR^{m\times n}$ and output $\mathbf{y}\in \bbR^{m}$ 
   \STATE {\bfseries Output:} a $k$-sized subset of $[1:n]$
   \STATE $\mathbf{l} \gets (0, \ldots, 0)$,\, $\mathbf{l}\in \reals^n$
    \FOR{each $i\in [1:n]$}
      \STATE $l_i \gets \sum_{j = 1}^{m}y_jA_{j,i}$ 
    \ENDFOR
    \STATE Sort $\mathbf{l}$ in decreasing order and let $\cS$ be the top $k$ indices.\\
    \STATE {\bfseries Return:} $\mathcal{\cS}$ 
\end{algorithmic}
\end{algorithm}


The convergence and sample complexity guarantees for the algorithm are shown for the case when each entry of $\vecA$ is chosen iid $\cN(0,1)$. Note that such a matrix satisfies the power constraint in \eqref{eq:power_constraint}. As we argued in Section~\ref{sec:intro}, for the unknown signal $\bx$,  the output $\by = \vecA{\bx}+\bz$ is correlated with each column $\vecA_i$ for $i\in \cS_{\bx}$ and uncorrelated with $\vecA_j$ for $j\notin \cS_{\bx}$. In particular, for large number of samples, when $i\in \cS_{\bx}$, the inner product $\ipr{\by}{\vecA_i}$ is close to $\bbE\insq{\ipr{\by}{\vecA_i}} = m$ (for linear regression) with high probability. On the other hand, $\ipr{\by}{\vecA_j}$ is close to $0$ for $j\notin \cS_{\bx}$.  Thus, $l_i$ for $i\in \cS_{\bx}$ will dominate over $l_j$ for $j\notin \cS_{\bx}$. This line of argument also works when the output is binary, though in this case $\bbE\insq{\ipr{{\by}}{\vecA_i}}$ for $i\in \cS_{\bx}$ is different. 
This is the main idea of Algorithm~\ref{alg:1}. We first present Theorem~\ref{thm:alg_general} for generalized linear measurements.
\begin{theorem}[Sample Complexity of Algorithm~\ref{alg:1} for GLMs]\label{thm:alg_general}
Suppose the GLM is such that for each $i\in [m]$, $y_i$ is a subgaussian random variable with subgaussian norm given by $\normi{{y_i}}_{\psi_2}$. For any $\bx$, suppose for some $L$, $\bbE\insq{g'(\vecA_i^T\bx)}\geq L\cdot
\normi{{y_i}}_{\psi_2}$   for all $i\in [m]$. Algorithm~\ref{alg:1} recovers the unknown signal with high probability if 
\begin{align}
m \geq \frac{C}{{\min\inb{L, L^2}}}(\log\inp{k}+\log\inp{n-k})\label{eq: alg_bound}
\end{align} where $C$ is some constant.
\end{theorem}
When $y_j$ is subgaussian, $y_j\vecA_{i,j}$ for any $i,j$ is a sub-exponential random variable. This observation allows us to use a concentration result for sub-exponential random variables to analyse the sample complexity. See Section~\ref{sec:proofs} for a detailed proof. 

As corollaries to Theorem~\ref{thm:alg_general}, we obtain the following sample complexity bounds for \bcs\ and \spl. These corollaries are proved in Appendix~\ref{proof:sec:alg}.
\begin{corollary}[Sample Complexity of Algorithm~\ref{alg:1} for \bcs]\label{thm:alg_bcs}
Algorithm~\ref{alg:1} recovers the unknown signal for \bcs\ with high probability if $m=O\inp{\inp{k+\sigma^2}(\log\inp{k}+\log\inp{n-k})}$.
\end{corollary}
\begin{corollary}[Sample Complexity of Algorithm~\ref{alg:1} for \spl]\label{thm:alg_spl}
Algorithm~\ref{alg:1} recovers the unknown signal for \spl\ if $m=O\inp{\inp{k+\sigma^2}(\log\inp{k}+\log\inp{n-k})}$.
\end{corollary}
 Interestingly, the sample complexity for both \bcs\ and \spl\ is the same. This can be explained by similar values of $L$, which result in similar rates of concentration of $l_i$'s around their expectation in both the cases. 
This also implies that in the regime where $m = O((k+\sigma^2)\log(n-k))$, having access to $\vecA_i^T \bx+z_i$ instead of $\sign{\vecA_i^T \bx+z_i}$, does not improve the sample complexity beyond constants.

Using Theorem~\ref{thm:alg_general}, we obtain the following corollary for logistic regression (see proof in Appendix~\ref{proof:sec:alg}).
\begin{corollary}[Sample Complexity of Algorithm~\ref{alg:1} for \logreg]\label{thm:alg_logreg}
Algorithm~\ref{alg:1} recovers the unknown signal for \logreg\ if $m=O\inp{\inp{k+1/\beta^2}\inp{\log{k}+\log\inp{n-k}}}$.
\end{corollary}
Comparing the sample complexity bounds of \bcs\ and \logreg, we notice that the sample complexity is similar except that   the noise variance $\sigma^2$ is replaced by $1/\beta^2$. This relationship is not surprising as a similar relationship was also present in the sample complexity bounds in \cite{hsu2024sample} (for logistic regression) and \cite{kuchelmeister2024finite} (for probit model). Note that, in the noiseless case, when $\beta\rightarrow \infty$ (or $\sigma = 0$ for \bcs), the sample complexity is $O(k\log{n})$, which is close to the simple counting lower bound of $k\log{n/k}$. On the other hand, when $\beta = 0$ (or $\sigma\rightarrow \infty$ for \bcs), $m\rightarrow \infty$, which makes intuitive sense as very high levels of noise render the output useless.

To compute the time complexity of the algorithm, notice that the for loop in step 2 takes $O(n\times m)$ time and step 4 takes $O(n\log{n})$ time. Thus, the computational complexity of the algorithm is $O(nm+n\log{n})$, which is $O((k+\sigma^2)n\log{n})$ for $m = O((k+\sigma^2)\log{n}$.
To compute the time complexity of the algorithm, notice that the for loop in step 2 takes $O(n\times m)$ time and step 4 takes $O(n\log{n})$ time. Thus, the computational complexity of the algorithm is $O(nm+n\log{n})$, which is $O((k+\sigma^2)n\log{n})$ for $m = O((k+\sigma^2)\log{n}$.

\subsection{Lower bounds on sample complexity}\label{sec:sample_compexity}
We establish a lower bound for generalized linear measurements using standard information-theoretic arguments based on Fano's inequality. While the upper bound in Theorem~\ref{thm:alg_general} is derived for the maximum probability of error over all  $k$-sparse vectors, the lower bound applies even in the weaker setting of the average probability of error, where 
$\bx$ is chosen uniformly at random.
\begin{theorem}[Lower bound for GLMs]\label{thm: lower_bdglm} Consider any  sensing matrix $\vecA$.
For a uniformly chosen $k$-sparse vector $\bx$, an algorithm $\phi$ satisfies $$\bbP\inp{\phi(\vecA, \by) \neq \bx}\leq \delta$$   only if the number of measurements $$m\geq \frac{k\log\inp{\frac{n}{k}}}{I}\inp{1 - \frac{h_2(\delta) + \delta k\log{n}}{k\log{n/k}}}$$ for some $I$ such that $I\geq {I(y_i; \bx|\vecA)}, \, i\in [m]$. In particular, when $y\in \inb{-1, 1}$, we have $\bbE\insq{\inp{g(\vecA_i^T\bx)}^2} \geq I(y_i, \bx|\vecA)$ where the expectation is over the randomness of $\vecA$ and $\bx$.
\end{theorem}
The lower bound can be interpreted in terms of a communication problem, where the input message $\bx$ is encoded to $\vecA\bx$. The decoding function takes in as input the encoding map $\vecA$ and the output vector $\by$ in order to recover $\bx$ with high probability. For optimal recovery, one needs at least $\frac{\text{message entropy}}{\text{capacity}}$ number of measurements (follows from noisy channel coding theorem~\cite{thomas2006elements}). In Theorem~\ref{thm: lower_bdglm}, the entropy of the message set $\log{n \choose k}\approx k\log{n/k}$ and the proxy for capacity is the upper bound on mutual information $I$. We provide a detailed proof of the theorem in  Section~\ref{sec:proofs}.

We first present lower bounds for \bcs\  and \logreg. The lower bound for \bcs\ is given for any sensing matrix $\vecA$ which satisfies the power constraint given by \eqref{eq:power_constraint}, whereas the one for \logreg\ is only for the special case when each entry of the sensing matrix is iid $\cN(0,1)$. Recall that \eqref{eq:power_constraint} holds in this case.  For \bcs\ (and \logreg\ respectively), we can use the upper bound of $\bbE\insq{\inp{g(\vecA_i^T\bx)}^2}$ on the mutual information term. The dependence of $\sigma^2$ (and $1/\beta^2$ respectively) requires careful bounding of this term, which is done in the formal proofs in Appendix~\ref{proof:sec:lower_bd}.

As mentioned earlier, we need at least $k\log\inp{n/k}$ measurements for \bcs and \logreg. This is because the entropy of a randomly chosen $k$-sparse vector is approximately $k\log\inp{n/k}$ and we learn at most one bit with each measurement. However, due to corruption with noise, we learn less than a bit of information about the unknown signal with each measurement. The information gain gets worse as the noise level increases. 
Our lower bounds make this reasoning explicit.  
\begin{corollary}[\bcs\ lower bound]\label{thm: lower_bd_bcs} Suppose, each row $\vecA_i, \, i\in [1:m]$ of the sensing matrix $\vecA$ satisfies the power constraint~\eqref{eq:power_constraint}.
For a uniformly chosen $k$-sparse vector $\bx$, an algorithm $\phi$ satisfies $$\bbP\inp{\phi(\vecA, {\by}) \neq \bx}\leq \delta$$ for the problem of $\bcs$ only if the number of measurements $$m\geq \frac{k+\sigma^2}{2}\log\inp{\frac{n}{k}}\inp{1 - \frac{h_2(\delta) + \delta k\log{n}}{k\log{n/k}}}.$$ 
\end{corollary}

\begin{corollary}[\logreg\ lower bound]\label{thm: lower_bd_log_reg} Consider a Gaussian  sensing matrix $\vecA$ where each entry is chosen iid $N(0,1)$.
For a uniformly chosen $k$-sparse vector $\bx$, an algorithm $\phi$ satisfies $$\bbP\inp{\phi(\vecA, \bw) \neq \bx}\leq \delta$$ for the problem of $\logreg$ only if the number of measurements $$m\geq \frac{1}{2}\inp{k+\frac{1}{\beta^2}}\log\inp{\frac{n}{k}}\inp{1 - \frac{h_2(\delta) + \delta k\log{n}}{k\log{n/k}}}.$$ 
\end{corollary}

Theorem~\ref{thm: lower_bdglm} also implies an information theoretic lower bound for \spl, which is presented below and proved in Appendix~\ref{proof:sec:lower_bd}. Note that the denominator term in the bound $\frac{1}{2}\log\inp{1+\frac{k}{\sigma^2}}$ is the capacity of a Gaussian channel with power constraint $k$ and noise variance $\sigma^2$. 
\begin{corollary}[\spl\ lower bound]\label{thm: spl_lower_bd_1}
Under the average power constraint \eqref{eq:power_constraint} on  $\vecA$, for a uniformly chosen $k$-sparse vector $\bx$, an algorithm $\phi$ satisfies $$\bbP\inp{\phi(\vecA, {\by}) \neq \bx}\leq \delta$$ only if the number of measurements
$$m\geq \frac{k\log\inp{\frac{n}{k}}-\inp{h_2(\delta) + \delta k\log{n}}}{\frac{1}{2}\log\inp{1+\frac{k}{\sigma^2}}}.$$
\end{corollary} 

\subsection{Tighter upper and lower bounds for \spl}\label{sec:tighter_bounds_spl}
We present information theoretic upper and lower bounds for \spl\ in this section. Similar to Section~\ref{sec:alg}, our upper bound is for the maximum probability of error, while the lower bounds hold even for the weaker criterion of average probability of error.

We first present an upper bound based on the maximum likelihood estimator (MLE) where  we  decode to $\hat{\bx}$ if, on output $\by$, 
\begin{align*}
\hat{\bx} = \argmax_{\stackrel{\bx\in \inb{0,1}^n}{\wh{\bx} = k}}\,\, p(\by|{\bx})
\end{align*} where $p(\by|{\bx})$ denotes the probability density function of $\by$ on input $\bx$.
\begin{theorem}[MLE upper bound for \spl]\label{thm:upper_bd_mle} Suppose  entries of the measurement matrix $\vecA$ are i.i.d. $\cN(0,1).$
The MLE  is correct with high probability if 
\begin{align}m\geq \max_{l\in[1:k]}  \frac{nN(l)}{\frac{1}{2}\log\inp{\frac{ l}{2\sigma^2}+1}}\label{eq:upper_bd_mle}
\end{align}where  $N(l):=  \frac{k}{n} h_2\inp{\frac{l}{k}} + (1-\frac{k}{n})h_2\inp{\frac{l}{n-k}}$. 
\end{theorem}
We prove the theorem in Appendix~\ref{proof:MLE}. The main proof idea involves analysing the probability that the output of the MLE is $2l$ Hamming distance away from the unknown signal $\bx$ for different values of $l\in [1:k]$ (assuming $k\leq n/2$). This depends on the number of such vectors (approximately $2^{nN(l)}$) and the probability that the MLE outputs a vector which is $2l$ Hamming distance away from $\bx$. 

Note that when $l = k\inp{1-\frac{k}{n}}$, $nN(l) = nh_2(k/n)\approx k\log{\frac{n}{k}}$ and $\log\inp{\frac{k\inp{1-k/n}}{2\sigma^2}+1}\leq \log\inp{\frac{k}{2\sigma^2}+1}$.
Thus, $m$ is at least $\frac{2k\log{n/k}}{\log\inp{\frac{k}{2\sigma^2}+1}}$ (see the bound for Corollary~\ref{thm: spl_lower_bd_1}). It is not immediately clear if this value of $l= k\inp{1-\frac{k}{n}}$ is the optimizer. However, for large $n$, this appears to be the case numerically as shown in Plot~\ref{plot:1}.

\begin{figure}[t]
\includegraphics[width=7cm]{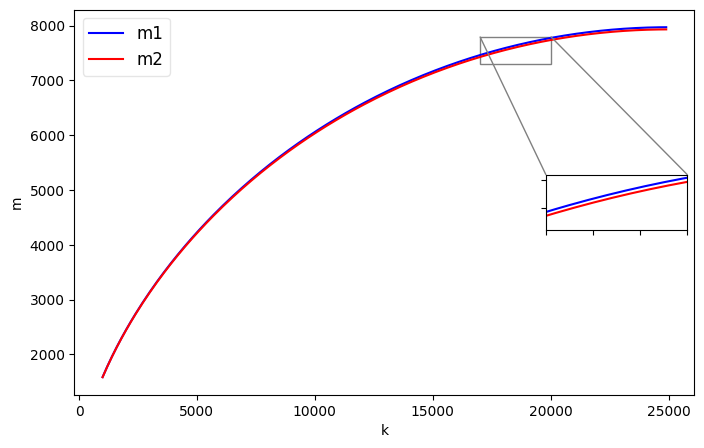}
\centering
\caption{The figure shows the plot of the MLE upper bound \eqref{eq:upper_bd_mle} (given by m1) for different values of $k$. This is displayed in blue color. A plot of $\frac{2nN(l)}{\log\inp{\frac{ l}{2\sigma^2}+1}}$ is also presented for $l = k\inp{1-\frac{k}{n}}$ in orange color, given by m2. A part of the plot is zoomed in to emphasize the closeness between the lines. In these plots,  $\sigma^2$ is set to 1,  $n$ is 50000 and $k$ ranges from 1000 to 25000 $(n/2)$. }\label{plot:1}
\end{figure}

Inspired by the MLE analysis, we derive a lower bound with the same structure as \eqref{eq:upper_bd_mle}. We generate the unknown signal $\bx$ using the following distribution: A vector $\tilde{\bx}$ is chosen uniformly at random from the set of all $k$-sparse vectors. Given $\tilde{\bx}$, the unknown input signal $\bx$ is chosen uniformly from the set of all $k$-sparse vector which are at a Hamming distance $2l$ from $\bx$. 
The lower bound is then obtained by computing upper and lower bounds on $I(\vecA, \by;\bx|\tilde{\bx})$.
We show this lower bound only for random matrices where each entry is chosen iid $\cN(0,1)$.
\begin{theorem}[\spl\ lower bound]\label{thm:lower_bd_spl}
If each entry of $\vecA$ is chosen iid $\cN(0,1)$, then for a uniformly chosen $k$-sparse vector $\bx$, an algorithm $\phi$ satisfies 
\begin{align}
    \bbP\inp{\phi(\vecA, {\by}) \neq \bx}\leq \delta\label{eq:spl_lower_bd_l}
\end{align}  only if the number of measurements $$m\geq \max_l\frac{nN(l) - 2\log{n}- h_2(\delta) - \delta k\log{n}}{\frac{1}{2}\log\inp{1+\frac{l}{\sigma^2}\inp{2-\frac{l}{k}}}} .$$
\end{theorem} The proof of Theorem~\ref{thm:lower_bd_spl} is given in Appendix~\ref{proof:MLE}.

If we choose $l = k\inp{1-\frac{k}{n}}$ in Theorem~\ref{thm:lower_bd_spl}, we recover corollary~\ref{thm: spl_lower_bd_1} for the special case of Gaussian design.



\section{Proofs}\label{sec:proofs}
\begin{proof}[Proof of Theorem~\ref{thm:alg_general}]
Consider any input $\bx$ and a sensing matrix $\vecA$ where each entry is chosen iid $\cN(0,1)$. Suppose ${\bx}$ is supported on $\cS\subseteq[1:n]$ where $|\cS|=k$. Let $\by = (y_1, \ldots, y_m)$. Consider the event
\begin{align*}
    \cF = \inb{\sum_{i=1}^m{ {y_i A_{i,j}}}> \sum_{i=1}^m{ {y_i A_{i,j'}}}\text{ for all } j \in \cS,  j'\in\cS^c}
\end{align*}
It is clear that under $\cF$, the algorithm is correct.
We will compute the probability of $\cF^c$.
\begin{align}
\bbP\inp{\cF^c} &= \bbP\inp{\bigcup_{j\in\cS}\bigcup_{j'\in \cS^c}\inb{\sum_{i=1}^m{ {y_i A_{i,j'}}}\geq \sum_{i=1}^m{ {y_i A_{i,j}}}}}\nonumber\\
&\leq \sum_{j\in\cS}\sum_{j'\in \cS^c}\bbP\inp{\sum_{i=1}^m{ {y_i A_{i,j'}}}\geq \sum_{i=1}^m{ {y_i A_{i,j}}}}\nonumber\\
&=\sum_{j\in\cS}\sum_{j'\in \cS^c}\bbP\inp{\sum_{i=1}^m{ \inp{y_i (A_{i,j'}-A_{i,j})}}\geq 0}\label{eq:log_prob_f^c1}
\end{align}

For any $i\in[1:m]$, $j\in\cS$ and $j'\in \cS^c$, we first compute $\bbE\insq{y_i (A_{i,j}-A_{i,j'})}$.
\begin{align}
\bbE\insq{y_i (A_{i,j}-A_{i,j'})} &= \bbE\insq{y_i A_{i,j}}-\bbE\insq{y_iA_{i,j'}}\nonumber\\
& \stackrel{(a)}{=} \bbE\insq{y_i A_{i,j}} \label{eq:log_expt11}\\
&\stackrel{(b)}{=}\frac{\bbE\insq{y_iA_{i,\cS}}}{k}\nonumber\\
& = \frac{\bbE\insq{y_i\vecA_{i}^T\bx}}{k}\nonumber\\
&= \frac{\bbE\insq{A_{i}^T\bx\,\bbE\insq{y_1|\vecA_i^T\bx}}}{k}\nonumber\\
& \stackrel{(c)}{=} \frac{\bbE\insq{A_{i}^T\bx g\inp{\vecA_i^T\bx}}}{k}\nonumber\\
& \stackrel{(d)}{=} {\bbE\insq{ g'\inp{\vecA_i^T\bx}}}:= E\label{eq:log_expt1_avg11}
\end{align} where $(a)$ follows from the fact that $y_i$ and  $A_{i,j'}$ are zero mean, independent random variables and $(b)$ follows by defining $A_{i,\cS} = \sum_{j\in \cS}A_{i,j}$ and noticing that the random variables $y_i A_{i,j}$ are identically distributed for all $j\in \cS$, $(c)$ follows from \eqref{eq:glm} and $(d)$ follows from Stein's lemma. 
\begin{align*}
\bbP&\inp{\sum_{i=1}^m{ \inp{y_i (A_{i,j'}-A_{i,j})}}\geq 0}\\
&=\bbP\inp{\sum_{i=1}^m{ \inp{y_i (A_{i,j}-A_{i,j'})}}\leq 0}\\
& = \bbP\inp{\sum_{i=1}^m{ \inp{y_i (A_{i,j}-A_{i,j'})}} - mE\leq -mE}\\
& \leq \bbP\inp{\left|\sum_{i=1}^m{ \inp{y_i (A_{i,j}-A_{i,j'})}} - mE\right|\geq mE}\\
\end{align*}
\vspace{-0.001cm}
To compute this, note that for all $i\in [1:m]$, $y_i$  is a subgaussian random variable and ${y_i}\inp{A_{i, j}-A_{i, j'}}$ being product of two subgaussian random variables is a subexponential random variable (see [Lemma 2.7.7]\cite{vershynin}). Note that $\bbE\insq{\sum_{i=1}^m{ \inp{y_i (A_{i,j}-A_{i,j'})}}} = mE$ where $E$ was defined in \eqref{eq:log_expt1_avg11}.
Also, 
\begin{align*}
&\hspace{-0.3cm}\normi{{y_i}\inp{A_{i, j}-A_{i, j'}} - 1}_{\psi_1}\\
&\stackrel{(a)}{\leq} C\normi{{y_i}\inp{A_{i, j}-A_{i, j'}}}_{\psi_1}\\
&\stackrel{(b)}{\leq} C\normi{{y_i}}_{\psi_2}\normi{\inp{A_{i, j}-A_{i, j'}}}_{\psi_2}\\
&\stackrel{(c)}{\leq} C\normi{{y_i}}_{\psi_2}2C'\\
& = C_1\normi{{y_i}}_{\psi_2} \qquad\qquad\text{ for some constant $C_1$}.
\end{align*} Here, $(a)$ follows from [Exercise 2.7.10]\cite{vershynin}, $(b)$ from [Lemma 2.7.7]\cite{vershynin} and $(c)$ from [Example 2.5.8]\cite{vershynin}.
With this
\begin{align*}
\bbP&\inp{\left|\sum_{i=1}^m{ \inp{y_i (A_{i,j}-A_{i,j'})}} - mE\right|\geq mE}\\
&\stackrel{(a)}{\leq} 2\exp\inp{-c\min\inp{\frac{m^2E^2}{mC_1^2\normi{{y_i}}_{\psi_2}^2}, \frac{mE}{C_1\normi{{y_i}}_{\psi_2}} }}\\
&\stackrel{(b)}{\leq} 2\exp\inp{-cm\min\inp{\frac{mL^2}{C_1^2}, \frac{mL}{C_1 }}}
\end{align*} where $(a)$ follows from [Theorem 2.8.1]\cite{vershynin} and $(b)$ follows from the assumption in the lemma that $\frac{E}{\normi{{y_i}}_{\psi_2}} = \frac{\bbE\insq{g'(\vecA_i^T\bx)}}{\normi{{y_i}}_{\psi_2}}\geq L$.
Thus, from \eqref{eq:log_prob_f^c1}, 
\begin{align*}
\bbP\inp{\cF^c}&\leq k(n-k)2\exp\inp{-C_2 m\min\inp{L^2, L}}\\
&\rightarrow 0 \text{ if }m\geq C_2\inp{\log{k}+\log\inp{n-k}}\frac{1}{\min\inp{L^2, L}}
\end{align*} for some constant $C_2$.
\end{proof}

\begin{proof}[Proof of Theorem~\ref{thm: lower_bdglm}]
Suppose $\bx$ is distributed uniformly on the set of all $k$-sparse binary vectors. Then,
\begin{align}
I(\vecA, \by; \bx)  &= H(\bx) - H(\bx|\vecA, \by)\nonumber\\
&\stackrel{(a)}{\geq} \log{n \choose k } - h_2(\delta) - \delta\log\inp{{n \choose k } + 1}\nonumber\\
&\geq k\log{n/k}- h_2(\delta)- \delta k\log\inp{n}\label{eq:log_lower_bd_bcy_1}
\end{align} where $(a)$ follows from Fano's inequality [Theorem~2.10.1]\cite{thomas2006elements}.
We also note that
\begin{align*}
 I(\vecA, \by; \bx) &= I(\vecA ; \bx)    +  I(\by; \bx|\vecA)\\
 &\stackrel{(a)}{ = }0 + I(\by; \bx|\vecA).
\end{align*}where $(a)$ holds because $\vecA$ and $\bx$ are independent. Let $y_{j\in[1:i-1]}$ denote $\inp{y_1, \ldots, y_{i-1}}$. 
\begin{align}
I&(\by; \bx|\vecA) = \sum_{i = 1}^{m}I(y_i; \bx|\vecA, y_{j\in[1:i-1]})\nonumber\\
& = \sum_{i = 1}^{m}\Big(H(y_i|\vecA, y_{j\in[1:i-1]})\nonumber\\
&\qquad- H(y_i| \bx,\vecA, y_{j\in[1:i-1]})\Big)\nonumber\\
& \stackrel{(a)}{\leq}\sum_{i = 1}^{m}\inp{H(y_i|\vecA)- H(y_i| \bx,\vecA)}\nonumber\\
& = \sum_{i = 1}^{m}I(y_i;\bx|\vecA)\nonumber\\
&\stackrel{(b)}{\leq} mI \label{eq:log_lower_bd_bg}
\end{align}where $(a)$ follows from $H(y_i|\vecA, y_{j\in[1:i-1]})\leq H(y_i|\vecA)$ and $H(y_i| \bx,\vecA, y_{j\in[1:i-1]}) = H(y_i| \bx,\vecA)$ as $y_i$ is conditionally independent of $y_{j\in[1:i-1]}$ conditioned on $\bx$ and $\vecA$ and $(b)$ follows from the assumption in the Theorem. Thus, from \eqref{eq:log_lower_bd_bcy_1} and \eqref{eq:log_lower_bd_bg},
\begin{align*}
mI &\geq  k\log\inp{n/k}\inp{1-\frac{h_2(\delta)+ \delta k\log\inp{n}}{k\log{n/k}}}
\end{align*}
This gives us the desired bound.

We can further simplify $I(y_i;\bx|\vecA)$ when $y_i\in \inb{-1,1}$, 
\begin{align*}
I(y_i;\bx|\vecA) &= H(y_i|\vecA)- H(y_i|\bx, \vecA)\\
&\stackrel{(a)}{\leq} 1- H(y_i|\bx, \vecA_i).
\end{align*}where $(a)$ holds because $H(y_i|\vecA)\leq H(y_i) =1$ and $y_i$ is conditionally independent of $(\vecA_1\ldots, \vecA_{i-1}, \vecA_{i+1}, \ldots, \vecA_{m})$ conditioned on $\vecA_i$ and $\bx$. Here $\vecA_i$, $i\in [1:m]$ denotes the $i^{\text{th}}$ row of the sensing matrix $\vecA$.

Suppose $\bx$ is fixed and $\bbP\inp{y_i = 1} = \frac{1}{2} + t$ for some $t\in [-1,1]$.  Then $\bbE\insq{y_i|\vecA_i} = 2t = g(\vecA_i^T\bx)$.
\begin{align*}
H(y_i|\vecA_i, \bx) &\stackrel{(a)}{=} \bbE\insq{h_2\inp{\frac{1}{2} + t}}\\
& \stackrel{(b)}{\geq} \bbE_{\bx}\insq{\bbE_{\vecA}\insq{4\inp{\frac{1}{2} + t}\inp{\frac{1}{2} - t}\Big|\bx}}\\
& = 1- \bbE_{\bx}\insq{\bbE\insq{\inp{2t}^2\Big|\bx}}\\
& = 1- \bbE_{\bx}\insq{\bbE\insq{\inp{g(\vecA_i^T\bx)}^2\Big|\bx}}\\
& = 1- \bbE_{\vecA, \bx}\insq{\inp{g(\vecA_i^T\bx)}^2}
\end{align*}
where in $(a)$, the expectation is over $\vecA$ and $\bx$. The inequality $(b)$ follows from [Theorem 1.2]\cite{topsoe2001bounds}. With this $I(y_i;\bx|\vecA)\leq \bbE\insq{\inp{g(\vecA_i^T\bx)}^2}$.
\end{proof}

\section{Conclusion and open problems}\label{sec:conclusion}
We analyze a simple algorithm (the ``average algorithm'' from \cite{vershyninPlan}  followed by `top-k' selection) for recovering sparse binary vectors from generalized linear measurements; along with an information theoretic lower bound. This gives optimal sample complexity characterization for \bcs\ and \logreg. On the other hand, the required number of measurements for the noisy linear case (\spl), which is $O((k+\sigma^2)\log{n})$, is as good as the sample complexity of any other known efficient algorithm for this problem, up to constants.  An interesting open problem is to find a design matrix and an efficient algorithm which requires less than $(k+\sigma^2)\log{n}$ samples for \spl. When the noise variance is zero, we show such an algorithm in  Remark~\ref{remark:noNoise}.

We also present almost matching information theoretic upper and lower bounds for \spl\ given by  \eqref{eq:upper_bd_mle} and \eqref{eq:spl_lower_bd_l} respectively. The bounds are in the form of an optimization problem. While we present numerical evidence which suggests that \eqref{eq:upper_bd_mle} is optimized by $l = k\inp{1-\frac{k}{n}}$, a formal proof is still missing. The bounds in   \eqref{eq:upper_bd_mle} and \eqref{eq:spl_lower_bd_l} also differ slightly by constants in the denominator, which seems to be a persistent gap in this problem.

\paragraph{Acknowledgment}
This work is supported in part by NSF awards 2217058 and 2112665. The authors would like to thank Krishna Narayanan who introduced them to the binary linear regression problem at the Simons Institute program on Error-correcting codes.

\bibliography{example_paper}
\bibliographystyle{alpha}

\newpage
\appendix
\onecolumn
\section{Proofs}\label{appendix:proofs}
\subsection{Missing proofs from Section~\ref{sec:alg}}\label{proof:sec:alg}
\begin{proof}[Proof of Corollary~\ref{thm:alg_bcs}]
We need to compute a lower bound $L$ on  $\frac{\bbE\insq{g'(\vecA_i^T\bx)}}{\normi{{y_i}}_{\psi_2}}$.  Instead of computing $\bbE\insq{g'(\vecA_i^T\bx)}$, we will compute $\bbE\insq{y_iA_{i,j}}$ for  any $j$ in the support of $\bx$. From \eqref{eq:log_expt11} and \eqref{eq:log_expt1_avg11}, we note that $\bbE\insq{y_iA_{i,j}} = \bbE\insq{g'(\vecA_i^T\bx)}$.
Also note that
$\bbE\insq{y_i A_{i,j}}= \bbE\insq{ A_{i,j}\bbE\insq{y_i| A_{i,j}}}$.

For any $\cU\subseteq[1:n]$, we denote $\sum_{l \in \cU}A_{i,l}$ by $A_{i,\cU}$.  For any $A_{i,j} = a$, 
\begin{align*}
\bbP&\inp{y_i = 1|A_{i,j} = a} \\
&= \bbP\inp{A_{i,\cS\setminus\inb{j}}+z_i\geq -a} \\
& = \bbP\inp{\frac{A_{i,\cS\setminus\inb{j}}+z_i}{\sqrt{k-1+\sigma^2}}\geq -\frac{a}{\sqrt{k-1+\sigma^2}}}\\
& = 1-\Phi\inp{-\frac{a}{\sqrt{k-1+\sigma^2}}}
\end{align*} where $\Phi(x) = \frac{1}{2\pi}\int_{-\infty}^{x}e^{-\frac{t^2}{2}}dt$ is the cumulative distribution function of the standard Gaussian distribution. Thus, 
$\bbP\inp{y_i = -1|A_{i,j} = a}  = \Phi\inp{-\frac{a}{\sqrt{k-1+\sigma^2}}}$ and 
$$\bbE\insq{y_i| A_{i,j}=a} = 1-2\Phi\inp{-\frac{a}{\sqrt{k-1+\sigma^2}}}.$$ We are now ready to compute $\bbE\insq{ A_{i,j}\bbE\insq{y_i| A_{i,j}}}$.
\begin{align}
\bbE&\insq{ A_{i,j}\bbE\insq{y_i| A_{i,j}}} \nonumber\\
&= \bbE\insq{ A_{i,j}\inp{1-2\Phi\inp{-\frac{A_{i,j}}{\sqrt{k-1+\sigma^2}}}}}\nonumber\\
&= \bbE\insq{ A_{i,j}}-2\bbE\insq{A_{i,j}\Phi\inp{-\frac{A_{i,j}}{\sqrt{k-1+\sigma^2}}}}\nonumber\\
&= 0-2\bbE\insq{A_{i,j}\Phi\inp{-\frac{A_{i,j}}{\sqrt{k-1+\sigma^2}}}}\label{eq:expt2}
\end{align}
\begin{align}
\bbE&\insq{A_{i,j}\Phi\inp{-\frac{A_{i,j}}{\sqrt{k-1+\sigma^2}}}}\nonumber\\
& = \int_{-\infty}^{\infty}a\frac{1}{\sqrt{2\pi}}e^{-\frac{a^2}{2}}\inp{\frac{1}{\sqrt{2\pi}}\int_{-\infty}^{-\frac{a}{\sqrt{k-1+\sigma^2}}}e^{-\frac{t^2}{2}}dt}da\nonumber\\
& = \frac{1}{2\pi}\int_{-\infty}^{\infty}\int_{-\infty}^{-\frac{a}{\sqrt{k-1+\sigma^2}}}ae^{-\frac{a^2}{2}}e^{-\frac{t^2}{2}}dt\,da\nonumber\\
& \stackrel{(a)}{=} \frac{1}{2\pi}\int_{-\infty}^{\infty}\int_{-\infty}^{-{t}{\sqrt{k-1+\sigma^2}}}ae^{-\frac{a^2}{2}}e^{-\frac{t^2}{2}}da\,dt\nonumber\\
& =\frac{1}{2\pi}\int_{-\infty}^{\infty}\inp{\int_{-\infty}^{-{t}{\sqrt{k-1+\sigma^2}}}ae^{-\frac{a^2}{2}}da}e^{-\frac{t^2}{2}}dt\nonumber\\
& = \frac{1}{2\pi}\int_{-\infty}^{\infty}\inp{-e^{-\frac{t^2(k-1+\sigma^2)}{2}}}e^{-\frac{t^2}{2}}dt\nonumber\\
& = -\frac{1}{\sqrt{2\pi\inp{k+\sigma^2}}}\int_{-\infty}^{\infty}\frac{\sqrt{k+\sigma^2}}{\sqrt{2\pi}}e^{-\frac{t^2(k+\sigma^2)}{2}}dt\nonumber\\
& = -\frac{1}{\sqrt{2\pi\inp{k+\sigma^2}}}\label{eq:expt3}
\end{align} where $(a)$ follows for change of variable formula for integration. From \eqref{eq:expt2} and \eqref{eq:expt3}, we have 
\begin{align}
\bbE\insq{y_i A_{i,j}} = \sqrt{\frac{2}{\pi}}\times\frac{1}{\sqrt{\inp{k+\sigma^2}}}.\label{eq:expt4}
\end{align} From [Example 2.5.8]\cite{vershynin}, we also note that $\normi{{y_i}}_{\psi_2} = 1$. Thus, $L = \sqrt{\frac{2}{\pi}}\times\frac{1}{\sqrt{\inp{k+\sigma^2}}}$ and $\min\inb{L, L^2} = L^2$, which when substituted in \eqref{eq: alg_bound}  gives the desired bound.
\end{proof}

\begin{proof}[Proof of Corollary~\ref{thm:alg_spl}]
We first note that  $\frac{\bbE\insq{g'(\vecA_i^T\bx)}}{\normi{{y_i}}_{\psi_2}} = \frac{\bbE\insq{y_iA_{i,j}}}{\normi{{y_i}}_{\psi_2}}$ for  any $j$ in the support of $\bx$. This follows from \eqref{eq:log_expt11} and \eqref{eq:log_expt1_avg11}. We first compute $\bbE\insq{y_iA_{i,j}}$, which is the same as $\bbE\insq{\inp{\vecA_i^T\bx+z_i}A_{i,j}}$ for \spl. 
Note that $\bbE\insq{\inp{\vecA_i^T\bx+z_i}\inp{A_{i, j}}} = \bbE\insq{A_{i, j}^2} = 1$.
Also, from [Example 2.5.8]\cite{vershynin} 
\begin{align*}
\normi{\inp{\vecA_i^T\bx+z_i}}_{\psi_2}\leq C\sqrt{k+\sigma^2}\\
\end{align*} for some constants $C$. With this,
\begin{align*}
\frac{\bbE\insq{g'(\vecA_i^T\bx)}}{\normi{{y_i}}_{\psi_2}}\geq \frac{1}{C\sqrt{k+\sigma^2}}: =L.
\end{align*}Note that $\min\inb{L, L^2} = L^2$, which when substituted in \eqref{eq: alg_bound} gives the desired bound.
\end{proof}

\begin{proof}[Proof of Corollary~\ref{thm:alg_logreg}]
We will first compute $g(\vecA_i^T\bx) = \bbE\insq{y_i|\vecA_i^T\bx}$ for \logreg.
\begin{align*}
g(\vecA_i^T\bx)&=\bbE\insq{y_i|\vecA_i^T\bx}\\
&= \frac{1}{1+e^{-\beta \vecA_i^T\bx}}-\frac{e^{-\beta \vecA_i^T\bx}}{1+e^{-\beta \vecA_i^T\bx}}\\
&=\frac{1-e^{-\beta \vecA_i^T\bx}}{1+e^{-\beta \vecA_i^T\bx}}\\
& \stackrel{(a)}{=} {\tanh\inp{\frac{\beta \vecA_i^T\bx}{2}}}
\end{align*}where $(a)$ uses the definition of $\tanh$. Then
\begin{align*}
\bbE\insq{g'(\vecA_i^T\bx)} &= \frac{\beta}{2}\bbE\insq{\frac{1}{{\text{cosh}}^2\inp{\frac{\beta \vecA_i^T\bx}{2}}}}\\
& \stackrel{(c)}{\geq}\frac{\beta}{2}\bbE\insq{e^{-\frac{\inp{\beta\vecA_i^T\bx}^2}{4}}}
\end{align*} where $(c)$ follows from the inequality $\text{cosh}(t)\leq e^{t^2/2}$ (see [Exercise 2.2.3]\cite{vershynin}). 


Now, we need to compute $\bbE\insq{e^{-\frac{\inp{\beta\vecA_i^T\bx}^2}{4}}}$ where $\vecA_i^T\bx\sim N(0,k)$. Let $\sigma_1 := \frac{1}{\frac{\beta^2}{2}+\frac{1}{k}}$. Then

\begin{align}
\bbE\insq{e^{-\frac{\inp{\beta\vecA_i^T\bx}^2}{4}}} &= \int_{-\infty}^{\infty}\frac{1}{\sqrt{2\pi k}}e^{-\beta^2a^2/4}e^{-a^2/2k} da\nonumber\\
& = \sqrt{\frac{\sigma_1}{k}}\int_{-\infty}^{\infty}\frac{1}{\sqrt{2\pi \sigma_1}}e^{-x^2/2\sigma_1} da\nonumber\\
& = \sqrt{\frac{\sigma_1}{k}}\nonumber\\
& = \sqrt{\frac{2}{2+\beta^2 k}}\label{eq:expectation_log}
\end{align}

Thus,
\begin{align*}
\bbE\insq{g'(\vecA_i^T\bx)}&\geq \frac{\beta}{2}\sqrt{\frac{2}{2+\beta^2 k}}\\
& = \frac{1}{2}\sqrt{\frac{2}{2/\beta^2+ k}}.
\end{align*}
From [Example 2.5.8]\cite{vershynin}, we also note that $\normi{{y_i}}_{\psi_2} = 1$. Thus, $L = \frac{1}{2}\sqrt{\frac{2}{2/\beta^2+ k}}$ and $\min\inp{L, L^2} = L^2$, which gives the desired bound.

\end{proof}

\subsection{Missing proofs from Section~\ref{sec:sample_compexity}}\label{proof:sec:lower_bd}
\begin{proof}[Proof of Corollary~\ref{thm: lower_bd_bcs}]
Consider a sensing matrix $\vecA$ which satisfies the power constraint \eqref{eq:power_constraint}. 

Here $\vecA_i$, $i\in [1:m]$ denotes the $i^{\text{th}}$ row of the sensing matrix $\vecA$.
For any realization $b\in \bbR$ of $\vecA_i^T\bx$, 
\begin{align*}
\bbP(y_i = 1|\vecA_i^T\bx = b) &= \bbP(z_i\geq -b)= \bbP\inp{\frac{z_i}{\sigma}\geq \frac{-b}{\sigma}}\\
& = \frac{1- \sign{b}}{2} + \sign{b}Q\inp{\frac{|b|}{\sigma}}.
\end{align*} 

For $a>0$, let $R(a):=\frac{1}{\sqrt{2\pi}}\int_{0}^{a}e^{-u^2/2}du$. Then $Q(a) = \frac{1}{2}-R(a)$. Suppose $\bx$ is fixed. Then,
\begin{align*}
g(\vecA_i^T\bx) & =\bbE\insq{y_i|\vecA}  = \bbE\insq{y_i|\vecA_i^T\bx}\\
& = \frac{1- \sign{\vecA_i^T\bx}}{2} + \sign{\vecA_i^T\bx}Q\inp{\frac{|\vecA_i^T\bx|}{\sigma}} - \inp{1-\inp{\frac{1- \sign{\vecA_i^T\bx}}{2} + \sign{\vecA_i^T\bx}Q\inp{\frac{|\vecA_i^T\bx|}{\sigma}}}}\\
& = \sign{\vecA_i^T\bx}\inp{1-2Q\inp{\frac{|\vecA_i^T\bx|}{\sigma}}}\\
& = \sign{\vecA_i^T\bx}\inp{2R\inp{\frac{|\vecA_i^T\bx|}{\sigma}}}
\end{align*}
For any $a>0$,
\begin{align*}
R\inp{a} &= \frac{1}{\sqrt{2\pi}}\int_{0}^{a}e^{-u^2/2}du\\
& \leq \frac{1}{\sqrt{2\pi}}\int_{0}^{a}1du = \frac{a}{\sqrt{2\pi}}.
\end{align*}
Thus, 
\begin{align*}
\bbE\insq{\inp{g\inp{\vecA_i^T\bx}^2}} &= \bbE\insq{\inp{2R\inp{\frac{|\vecA_i^T\bx|}{\sigma}}}^2}\\
&\leq \bbE\insq{4\inp{\frac{\vecA_i^T\bx}{\sqrt{2\pi}\sigma}}^2}\\
&\stackrel{(a)}{\leq}\frac{2k}{\pi\sigma^2} 
\end{align*}where $(a)$ follows from the power constraint $\bbE\insq{\inp{\vecA_i^T\bx}^2}\leq k$ (see \eqref{eq:power_constraint}). This holds for any $\bx$, including a randomly chosen sparse vector.
Thus,
\begin{align}
m &\geq  \frac{\pi\sigma^2}{2k}k\log\inp{n/k}\inp{1-\frac{h_2(\delta)+ \delta k\log\inp{n}}{k\log{n/k}}}\nonumber\\
& \geq \sigma^2\log\inp{n/k}\inp{1-\frac{h_2(\delta)+ \delta k\log\inp{n}}{k\log{n/k}}}\label{eq:final_bound_bcs1}
\end{align}
On the other hand, $I(y_i;\bx|\vecA)\leq 1$. Thus,
\begin{align}
k\log&\inp{n/k}\inp{1-\frac{h_2(\delta)+ \delta k\log\inp{n}}{k\log{n/k}}}\nonumber\\
&\leq \sum_{i = 1}^{m}I(y_i;\bx|\vecA)\leq \sum_{i = 1}^{m}H(y_i|\vecA)\nonumber\\
&\leq m.\label{eq:final_bound_bcs2}
\end{align}
Combining \eqref{eq:final_bound_bcs1} and \eqref{eq:final_bound_bcs2}, we get the desired bound.
\end{proof}

\begin{proof}[Proof of Corollary~\ref{thm: lower_bd_log_reg}]
Consider a Gaussian sensing matrix $\vecA$. Suppose $\bx$ is distributed uniformly on the set of all $k$-sparse binary vectors. 

Suppose $t = \frac{1}{2}\tanh{\frac{\beta\vecA_i^T\bx}{2}} \inp{=\frac{(1-e^{-\beta \vecA_i^T\bx}}{2\inp{1+e^{-\beta \vecA_i^T\bx}}}}$. Then, 
\begin{align*}
\frac{1}{1+e^{-\beta \vecA_i^T\bx}} &= \frac{1}{2}+t\text{ and }\\
1-\frac{1}{1+e^{-\beta \vecA_i^T\bx}} &= \frac{1}{2}-t
\end{align*}

With this,
\begin{align*}
\bbE\insq{\inp{g\inp{\vecA_i^T\bx}^2}} &= \bbE\insq{4t^2}\\
& = \bbE\insq{\inp{\tanh{\frac{\beta\vecA_i^T\bx}{2}}}^2}
\end{align*}

Note that,
\begin{align*}
\bbE\insq{\inp{\tanh{\frac{\beta\vecA_i^T\bx}{2}}}^2} = 1-\bbE\insq{\inp{\text{sech}{\frac{\beta\vecA_i^T\bx}{2}}}^2}
\end{align*} and
\begin{align*}
\bbE\insq{\inp{\text{sech}{\frac{\beta\vecA_i^T\bx}{2}}}^2} &= \bbE\insq{\frac{1}{\inp{\text{cosh}{\frac{\beta\vecA_i^T\bx}{2}}}^2}}\\
&\stackrel{(a)}{\geq} \bbE\insq{e^{-\inp{\beta\vecA_i^T\bx/2}^2}}\\
&\stackrel{(b)}{=}\sqrt{\frac{1}{1+\beta^2 k/2}}\\
&\stackrel{(c)}{\geq}1-\frac{\beta^2k}{2}
\end{align*} where $(a)$ follows from the inequality $\text{cosh}(t)\leq e^{t^2/2}$ (see [Exercise 2.2.3]\cite{vershynin}),  $(b)$ follows from \eqref{eq:expectation_log} and $(c)$ holds because $1-\frac{x}{2}\leq \frac{1}{\sqrt{1+x}}$ for any $x\geq 0$.
Thus, 
\begin{align*}
\bbE\insq{\inp{g\inp{\vecA_i^T\bx}^2}} \leq \frac{\beta^2k}{2}.
\end{align*}
This implies that
\begin{align*}
m\frac{\beta^2k}{2} &\geq  k\log\inp{n/k}\inp{1-\frac{h_2(\delta)+ \delta k\log\inp{n}}{k\log{n/k}}}
\end{align*}
Thus,
\begin{align}
m&\geq \frac{2}{\beta^2}\log\inp{n/k}\inp{1-\frac{h_2(\delta)+ \delta k\log\inp{n}}{k\log{n/k}}}\nonumber\\
&\geq \frac{1}{\beta^2}\log\inp{n/k}\inp{1-\frac{h_2(\delta)+ \delta k\log\inp{n}}{k\log{n/k}}}\label{eq:log_final_bound_bcs1}
\end{align}
We also  know that for any $i$, $I(y_i;\bx|\vecA)\leq H(y_i|\vecA)\leq 1$. Thus, we also obtain that
\begin{align}
m\geq k\log\inp{n/k}\inp{1-\frac{h_2(\delta)+ \delta k\log\inp{n}}{k\log{n/k}}}\label{eq:log_final_bound_bcs2}
\end{align}
Combining \eqref{eq:log_final_bound_bcs1} and \eqref{eq:log_final_bound_bcs2}, we get the desired bound.
\end{proof}

\begin{proof}[Proof of Corollary~\ref{thm: spl_lower_bd_1}]
Suppose $\bx$ is generated uniformly at random from the set of all $k$-sparse vectors and $\vecA$ is any sensing matrix which satisfies the power constraint given by \eqref{eq:power_constraint}. Then,
\begin{align*}
I&(y_i;\bx| \vecA) = h(y_i|\vecA)-h(y_i|\bx,\vecA)\\
&\leq \inp{h(y_i)-h(\mathbf{A}_i^T\bx + z_i|\bx,  \vecA)}\\
& = h(y_i) - h(z_i)\\
&\leq \inp{h(w_i)-h(z_i)}
\end{align*} where in the last inequality, $w_i\sim\cN\inp{0, \sigma^2_w}$ where $\mathsf{Var}(y_i)\leq \sigma^2_w$. We will now compute an upper bound on $\mathsf{Var}(y_i)$.
\begin{align*}
\mathsf{Var}(y_i) &\leq \bbE\insq{\inp{\vecA_i^T\bx+z_i}^2} = \bbE\insq{\inp{\vecA_i^T\bx}^2} + \sigma^2\\
&\leq k + \sigma^2
\end{align*} 
Thus, we have
\begin{align}
\inp{h(w_i)-h(z_i)}=\frac{1}{2}\log\inp{\frac{{k}}{\sigma^2}+1}.\label{eq:new2}
\end{align}
With this, we conclude that
\begin{align*}
m&\geq \frac{2k\log\inp{\frac{n}{k}}- h_2(\delta) - \delta k\log{n}}{\frac{1}{2}\log\inp{\frac{{k}}{\sigma^2}+1}}.
\end{align*}

\end{proof}

\subsection{Missing proofs from Section~\ref{sec:tighter_bounds_spl}}\label{proof:MLE}
\begin{proof}[Proof of Theorem~\ref{thm:upper_bd_mle}]
We consider a sensing matrix $\vecA$ where each entry is chosen iid $\cN(0,1)$. let $\cX_k$ denote the set of all $k$-sparse binary vectors. That is $\cX_k = \inb{{\bx'}\in \inb{0,1}^n: \wh{\bx} = k}$.  We  decode to $\hat{\bx}$ if, on output $\by$,
\begin{align*}
\hat{\bx} = \argmax_{\bx'\in \cX_k}\,\, p(\by|{\bx'})
\end{align*} where $p(\by|{\bx'})$ is the probability density function of $\by$ on input $\bx'$. 
 We assume that $k\leq n/2$. Suppose unknown signal is $\bx$. The error event $\cE$ is 
\begin{align*}
\cE = \inb{\by: \exists \tilde{\bx}\neq \bx \text{ such that }p(\by|\tilde{\bx})>p(\by|{\bx})}
\end{align*}
Then
\begin{align*}
\bbP(\cE) \leq \sum_{l = 1}^k\sum_{\substack{\bx'\in \cX_k:\\\h{\bx}{\bx'}= 2l}}\bbP\inp{p(\by|\tilde{\bx})>p(\by|{\bx})}
\end{align*}
Suppose $\bx$ has support on $\cS \subset [1:n], |\cS| = k$ and $\tilde{\bx}$ has support on $\cU \subset [1:n], |\cU| = k$. Then, conditioned on $\bx$, $y_r$ is generated from $\sum_{i\in \cS}A_{r, i}$ which we denote by $A_{r, \cS}$. That is, $y_r = A_{r, \cS} + z_r$ where $A_{r, \cS}\sim \cN(0, k)$. Similarly, conditioned on $\tilde{\bx}$, $y_r = A_{r, \cU} + z_r$ for $A_{r, \cU}:= \sum_{i\in \cU}A_{r, i}$ where $A_{r, \cU}\sim \cN(0, k)$.  For any $l\in [1:k]$, computing $\bbP\inp{p(\by|\tilde{\bx})>p(\by|{\bx})}$, we have
\begin{align*}
\bbP(p(\by|&\tilde{\bx})>p(\by|{\bx}))\\
&= \bbP\inp{\log{\frac{p(\by|\tilde{\bx})}{p(\by|{\bx})}}>0}\\
&=  \bbP\inp{{\sum_{r = 1}^{m}\log{\frac{p(y_r|A_{r, \cU})}{p(y_r|A_{r, \cS})}}}>0}\\
&= \bbP\Bigg(\sum_{r = 1}^{m}(A_{r, \cU}-A_{r, \cS})y_r>\\
&\quad\quad\sum_{r = 1}^{m}(A_{r, \cU}-A_{r, \cS})\frac{(A_{r, \cU}+A_{r, \cS})}{2}\Bigg)
\end{align*}Using the fact that $y_r = A_{r, \cS\setminus\cU}+A_{r, \cS\cap\cU}+z_r$, we have
\begin{align*}
\bbP&\inp{p(\by|\tilde{\bx})>p(\by|{\bx})}\\
&=  \bbP\Big(\sum_{r = 1}^{m}(A_{r, \cU\setminus\cS}-A_{r, \cS\setminus\cU})(A_{r, \cS\setminus\cU}+A_{r, \cS\cap\cU}+z_r)>\\
&\quad\sum_{r = 1}^{m}(A_{r, \cU\setminus\cS}-A_{r, \cS\setminus\cU})\frac{A_{r, \cU\setminus\cS}+A_{r, \cS\setminus\cU}+2A_{r, \cU\cap\cS}}{2}\Big)\\
&=  \bbP\Big(\sum_{r = 1}^{m}(A_{r, \cU\setminus\cS}-A_{r, \cS\setminus\cU})z_r)>\\
&\qquad\sum_{r = 1}^{m}\frac{A^2_{r, \cU\setminus\cS}}{2}-\frac{A^2_{r, \cS\setminus\cU}}{2}-A_{r, \cU\setminus\cS}A_{r, \cS\setminus\cU}+A^2_{r, \cS\setminus\cU}\Big)\\
&=  \bbP\Bigg(\sum_{r = 1}^{m}(A_{r, \cU\setminus\cS}-A_{r, \cS\setminus\cU})z_r)>\\
&\quad\qquad\qquad\frac{\sum_{r = 1}^{m}\inp{A_{r, \cU\setminus\cS}-A_{r, \cS\setminus\cU}}^2}{2}\Bigg)\\
&= \bbP\Bigg(\frac{\sum_{r = 1}^{m}(A_{r, \cU\setminus\cS}-A_{r, \cS\setminus\cU})z_r}{\sqrt{\sum_{r = 1}^{m}\inp{A_{r, \cU\setminus\cS}-A_{r, \cS\setminus\cU}}^2}\sigma})\\
 &\qquad \qquad>\frac{\sqrt{\sum_{r = 1}^{m}\inp{A_{r, \cU\setminus\cS}-A_{r, \cS\setminus\cU}}^2}}{2\sigma}\Bigg).
\end{align*}
Let $b_r = A_{r, \cU\setminus\cS}-A_{r, \cS\setminus\cU}$. Note that $b_r\sim \cN(0, 2l)$. Let $\bb = \inp{b_1,  \ldots, b_m}$. Let $\be = (e_1, \ldots, e_m)$ denote the realization of $\bb$. Then,
\begin{align*}
& \bbP\left(\frac{\sum_{r = 1}^{m}(A_{r, \cU\setminus\cS}-A_{r, \cS\setminus\cU})z_r}{\sqrt{\sum_{r = 1}^{m}\inp{A_{r, \cU\setminus\cS}-A_{r, \cS\setminus\cU}}^2}\sigma}>\right.\\
&\qquad\qquad\qquad\qquad\left.\frac{\sqrt{\sum_{r = 1}^{m}\inp{A_{r, \cU\setminus\cS}-A_{r, \cS\setminus\cU}}^2}}{2\sigma}\right)\\
& = \bbP\inp{\frac{\sum_{r = 1}^{m}b_{r}z_r}{\sqrt{\sum_{r = 1}^{m}\inp{b_{r}}^2}\sigma})>\frac{\sqrt{\sum_{r = 1}^{m}\inp{b_r}^2}}{2\sigma}}\\
& \stackrel{(a)}{=} \int p_{\bb}(\be) \bbP\inp{\frac{\sum_{r = 1}^{m}e_rz_r}{\sqrt{\sum_{r = 1}^{m}\inp{e_r}^2}\sigma})>\frac{\sqrt{\sum_{r = 1}^{m}\inp{e_r}^2}}{2\sigma}}d\be\\
&= \int p_{\bb}(\be)Q\inp{\frac{\sqrt{\sum_{r = 1}^{m}\inp{e_r}^2}}{2\sigma}}d\be
\end{align*} where in $(a)$, $p_{\bb}(\be)$ denotes the density of $\bb$ at $\be$ and $d\be$ is shorthand for $de_1 de_2\ldots de_m$. To analyse this further, we use the upper bound $Q(x)\leq \frac{1}{2}e^{-x^2/2}$.
\begin{align*}
\int &p_{\bb}(\be)Q\inp{\frac{\sqrt{\sum_{r = 1}^{m}\inp{e_r}^2}}{2\sigma}}d\be\\
& = \int\frac{1}{\inp{2\pi\cdot 2l}^{m/2}}2^{\inp{-\frac{\sum_{r=1}^{m}{e_r^2}}{2l}}}2^{\inp{-\frac{{\sum_{r = 1}^{m}{e_r}^2}}{8\sigma^2}}}d\be\\
& = \int\frac{1}{\inp{2\pi\cdot 2l}^{m/2}}2^{\inp{-\sum_{r=1}^{m}{e_r^2}\inp{\frac{1}{2l}+\frac{1}{8\sigma^2}}}}d\be\\
& = \frac{1}{\inp{\frac{1}{2l}+\frac{1}{8\sigma^2}}^{m/2}\inp{2l}^{m/2}}\\
&\quad\int\frac{1}{{\inp{2\pi}^{m/2}}}\inp{\frac{1}{2l}+\frac{1}{8\sigma^2}}^{m/2}2^{\inp{-\sum_{r=1}^{m}{e_r^2}\inp{\frac{1}{2l}+\frac{1}{8\sigma^2}}}}d\be\\
& = \inp{\frac{1}{1+\frac{l}{2\sigma^2}}}^{m/2}\\
& = 2^{\inp{-\frac{m}{2}\log\inp{1+\frac{l}{2\sigma^2}}}}
\end{align*}
Next, we observe that 
\begin{align*}
&\inl{\inb{\bx'\in \cX_k:\h{\bx}{\bx'}= 2l}} = {k\choose l}{n-k \choose l}   \\
&\qquad\stackrel{(a)}{\leq} 2^{kh_2\inp{\frac{l}{k}}}2^{(n-k)h_2\inp{\frac{l}{n-k}}}\\
&\qquad = 2^{n\inp{\frac{k}{n}h_2\inp{\frac{l}{k}} + \frac{(n-k)}{n}h_2\inp{\frac{l}{(n-k)}}}}\\
&\qquad = 2^{nN(l)}.
\end{align*}where $(a)$ uses the inequality ${n \choose k}\leq 2^{nh_2(k/n)}$. 
Then,
\begin{align*}
\bbP(\cE) \leq &\sum_{l = 1}^k2^{nN(l)}2^{\inp{-\frac{m}{2}\log\inp{1+\frac{l}{2\sigma^2}}}}\\
&\rightarrow 0 \text{ if }m\geq \max_l \frac{2nN(l)}{\log\inp{1+\frac{l}{2\sigma^2}}}
\end{align*}

\end{proof}

\begin{proof}[Proof of Theorem~\ref{thm:lower_bd_spl}]
    We consider a joint distribution given by the following process. $\tilde{\bx}$ is generated uniformly at random from the set of all $k$-sparse vectors. Given $\tilde{\bx}$, the unknown signal $\bx$ is chosen uniformly at random from the set of all vectors which are at a Hamming distance $2l$ from $\tilde{\bx}$ for some $l\in [1:k]$ (assuming $k\leq n/2$). 
    We will denote the realization of $\tilde{\bx}$ by $\bar{\bx}$ and the realization of $\bx$ by $\hat{\bx}$. With this, given $\tilde{\bx} = \bar{\bx}$,
\begin{align*}
\bbP\inp{\bx = \hat{\bx}|\tilde{\bx} = \bar{\bx}} = \frac{1}{{k \choose l}{n-k \choose l}}.    
\end{align*}
Note that the marginal distribution of $\bx$ is uniform over the set of all $k$-sparse vectors.

We will be using the below set of equations in our further analysis. For $\bx = \inp{x_1, \ldots, x_n}$,  any $j, l\in \cS_{\tilde{\bx}}$ where $j\neq l$, we have
\begin{align}
&\bbP\inp{x_j = 1|\tilde{\bx} = \bar{\bx}} = \frac{{k-1\choose l}{n-k \choose l}}{{k \choose l}{n-k \choose l}} = \frac{k-l}{k},\label{eq:1}\text{ and }\\
&\bbP\inp{x_j = x_l = 1|\tilde{\bx} = \bar{\bx}} = \frac{{k-2\choose l}{n-k \choose l}}{{k \choose l}{n-k \choose l}} = \inp{\frac{k-l}{k}}\inp{\frac{k-l-1}{k-1}},\label{eq:2}
\end{align}
For any sensing matrix $\vecA$, output vector $\by$ and an unknown signal $\bx$ generated from $\tilde{\bx}$ using the above process, we have
\begin{align}
I&(\vecA, \by;\bx|\tilde{\bx}) = H(\bx|\tilde{\bx})-H(\bx|\vecA, \by,\tilde{\bx})\nonumber\\
&\geq H(\bx|\tilde{\bx})-H(\bx|\vecA, \by)\nonumber\\
 &\stackrel{(a)}\geq H(\bx|\tilde{\bx}) - h_2(\delta) - \delta\log{{n \choose k}}\nonumber\\
 & = \sum_{\tilde{\bx}}\bbP\inp{\tilde{\bx} = \bar{\bx}} H(\bx|\tilde{\bx} = \bar{\bx}) - h_2(\delta) - \delta\log{{n \choose k}}\nonumber\\
 & \stackrel{(b)}{\geq} \sum_{\tilde{\bx}}\bbP\inp{\tilde{\bx} = \bar{\bx}} \log{k \choose l}{n-k \choose l} - h_2(\delta) - \delta k\log{n}\nonumber\\
  & \stackrel{(c)}{\geq} kh_2\inp{\frac{l}{k}} + (n-k)h_2\inp{\frac{l}{n-k}}-\log\inp{k+1}\nonumber\\
  &\qquad-\log\inp{n-k+1} - h_2(\delta) - \delta k\log{n}\nonumber\\
  & \stackrel{(d)}{\geq}nN(l) - 2\log{n}- h_2(\delta) - \delta k\log{n}\label{eq:fano_spl}
\end{align} where $(a)$ follows from [Theorem~2.10.1]\cite{thomas2006elements}, $(b)$ follows from ${n \choose k}\leq n^k$, $(c)$ follows from ${n \choose k}\geq \frac{1}{n+1}2^{nh_2(k/n)}$ ([Theorem 11.1.3]\cite{thomas2006elements}) where $h_2$ is the binary entropy function and $(d)$ follows by defining $N(l) = \frac{k}{n}h_2\inp{\frac{l}{k}} + (1-\frac{k}{n})h_2\inp{\frac{l}{n-k}}$.

Next, we will compute an upper bound on $I(\vecA, \by;\bx|\tilde{\bx})$.
\begin{align*}
I(\vecA, \by;\bx|\tilde{\bx}) &= I(\vecA;\bx|\tilde{\bx}) + I( \by;\bx|\vecA,\tilde{\bx})\\
&\stackrel{(a)}{=} 0 + I( \by;\bx|\vecA,\tilde{\bx})\\
&\stackrel{(b)}{=}\sum_{i=1}^{m}I(y_i;\bx|\vecA, y_{j\in [1:i-1]},\tilde{\bx})
\end{align*} where $(a)$ follows because $\vecA$ is independent of both $\bx$ and $\tilde{\bx}$. In particular, $\vecA$ is conditionally independent of $\bx$ conditioned on $\tilde{\bx}$. Here, $(b)$ follows from chain rule for mutual information where $y_{j\in [1:i-1]}$ denotes $(y_1, \ldots, y_{i-1})$.

Suppose $h(\cdot)$ denotes the differential entropy of a continuous random variable. For any $i\in [1:m]$,
\begin{align*}
I&(y_i;\bx|\vecA, y_{j\in [1:i-1]},\tilde{\bx})\\
& = {h(y_i|\vecA, y_{j\in [1:i-1]},\tilde{\bx})-h(y_i|\bx,\vecA, y_{j\in [1:i-1]},\tilde{\bx})}\\
&\stackrel{(a)}{\leq} {h(y_i|\vecA_{i, j\in \cS_{\tilde{\bx}}}, \tilde{\bx})-h(\vecA_i^{T}\bx + z_i|\vecA,\bx, Y^{i-1},\tilde{\bx})}\\
&={h(y_i|\vecA_{i, j\in \cS_{\tilde{\bx}}}, \tilde{\bx})-h(z_i)}\\
&={h(y_i|\vecA_{i, j\in \cS_{\tilde{\bx}}}, \tilde{\bx})-\frac{1}{2}\log\inp{2\pi e \sigma^2}}
\end{align*} where in $(a)$, we use $\vecA_{i, j\in \cS_{\tilde{\bx}}}$ to denote the set of elements $A_{i, j}$ for $j\in \cS_{\tilde{\bx}}$.
Conditioned on $\tilde{\bx}= \bar{\bx}$ and $ \vecA_{i, j\in \cS_{\tilde{\bx}}} = \mathbf{a}_{i, j\in \cS_{\tilde{\bx}}} $,

\begin{align*}
h&(y_i|\tilde{\bx}= \bar{\bx}, \vecA_{i, j\in \cS_{\tilde{\bx}}}=\mathbf{a}_{i, j\in \cS_{\tilde{\bx}}}) \\
&\stackrel{(a)}{=} h\inp{y_i-\bbE\insq{y_i|\tilde{\bx}= \bar{\bx}, \vecA_{\tilde{\bx}}=\mathbf{a}_{i, j\in \cS_{\tilde{\bx}}}}\Big|\tilde{\bx}= \bar{\bx}, \vecA_{\tilde{\bx}}=\mathbf{a}_{i, j\in \cS_{\tilde{\bx}}}}\\
&\stackrel{(b)}{\leq} h(W_{\mathbf{a}_{i, j\in \cS_{\tilde{\bx}}}})\\
\end{align*} where $(a)$ follows by noting that differential entropy does not change by centering ([Theorem~8.6.3]\cite{thomas2006elements}) and $(b)$ follows for  $W_{i, \tilde{\bx}}\sim\cN\inp{0, \sigma_w^2}$ where $\sigma_w^2\leq \mathsf{Var}(y_i-\bbE\insq{y_i|\tilde{\bx}= \bar{\bx}, \vecA_{i, j\in \cS_{\tilde{\bx}}}=\mathbf{a}_{i, j\in \cS_{\tilde{\bx}}}}\Big|\tilde{\bx}= \bar{\bx}, \vecA_{i, j\in \cS_{\tilde{\bx}}} = \mathbf{a}_{i, j\in \cS_{\tilde{\bx}}})$ from  the fact that for the same variance a Gaussian random variable maximizes the differential entropy and it increasing with increasing variance ([Theorem 8.6.5 and Example 8.1.2]\cite{thomas2006elements}).

Recall that each entry of $\vecA$ is chosen iid $\cN(0,1)$. In that case, 
$\mathsf{Var}\inp{y_i-\bbE\insq{y_i|\tilde{\bx}= \bar{\bx}, \vecA_{i, j\in \cS_{\tilde{\bx}}}=\mathbf{a}_{i, j\in \cS_{\tilde{\bx}}}}}$ conditioned on $\tilde{\bx}= \bar{\bx}$ and $\vecA_{\tilde{\bx}} = \mathbf{a}_{i, j\in \cS_{\tilde{\bx}}}$ is   given by $\bbE\insq{\inp{y_i}^2|\tilde{\bx}= \bar{\bx}, \vecA_{i, j\in \cS_{\tilde{\bx}}} = \mathbf{a}_{i, j\in \cS_{\tilde{\bx}}}} - \inp{\bbE\insq{y_i|\tilde{\bx}= \bar{\bx}, \vecA_{i, j\in \cS_{\tilde{\bx}}} = \mathbf{a}_{i, j\in \cS_{\tilde{\bx}}}}}^2$. We first analyse the first term.
\begin{align*}
 \bbE&\insq{\inp{y_i}^2|\tilde{\bx}= \bar{\bx}, \vecA_{i, j\in \cS_{\tilde{\bx}}} = \mathbf{a}_{i, j\in \cS_{\tilde{\bx}}}} \\
 & = \bbE\insq{\bbE\insq{\inp{\vecA_i^T\bx + z_i}^2|\tilde{\bx}= \bar{\bx},\vecA_{i, j\in \cS_{\tilde{\bx}}} = \mathbf{a}_{i, j\in \cS_{\tilde{\bx}}}, \bx}}
\end{align*}
For any $\bx = \hat{\bx}$,
\begin{align*}
\bbE&\insq{\inp{\vecA_i^T\bx + z_i}^2|\tilde{\bx}= \bar{\bx},\vecA_{\tilde{\bx}} = \mathbf{a}_{i, j\in \cS_{\tilde{\bx}}}, \bx = \hat{\bx}}\\
& \stackrel{(a)}{=} l+\sigma^2 + \inp{\mathbf{a}_{i, S_{\bx}\cap S_{\tilde{\bx}}}}^2
\end{align*} where $(a)$ holds because conditioned on $\vecA_{\tilde{\bx}} = \mathbf{a}_{i, j\in \cS_{\tilde{\bx}}}$ and $\bx = \hat{\bx}$, the random variable $\vecA_i^T\bx + z_i = \mathbf{A}_{i, S_{\bx}\setminus S_{\tilde{\bx}}} + \mathbf{a}_{i, S_{\bx}\cap S_{\tilde{\bx}}} + z_i$ and $|S_{\bx}\setminus S_{\tilde{\bx}}| = l$.

Similarly, 
we can analyze the second term.
\begin{align*}
\bbE&\insq{y_i|\tilde{\bx}= \bar{\bx}, \vecA_{i, j\in \cS_{\tilde{\bx}}} = \mathbf{a}_{i, j\in \cS_{\tilde{\bx}}}} \\
 &= \bbE\insq{\vecA_i^T\bx + z_i|\tilde{\bx}= \bar{\bx}, \vecA_{i, j\in \cS_{\tilde{\bx}}} = \mathbf{a}_{i, j\in \cS_{\tilde{\bx}}}}\\
 & = \bbE\insq{\bbE\insq{\vecA_i^T\bx + z_i|\tilde{\bx}= \bar{\bx}, \vecA_{i, j\in \cS_{\tilde{\bx}}} = \mathbf{a}_{i, j\in \cS_{\tilde{\bx}}},\bx}}
\end{align*}
For any $\bx = \hat{\bx}$,
\begin{align*}
\bbE&\insq{\vecA_i^T\bx + z_i|\tilde{\bx}= \bar{\bx}, \vecA_{i, j\in \cS_{\tilde{\bx}}} = \mathbf{a}_{i, j\in \cS_{\tilde{\bx}}}, \bx = \hat{\bx}}\\
& = \mathbf{a}_{i, S_{\bx}\cap S_{\tilde{\bx}}}
\end{align*} and
\begin{align*}
\bbE\insq{\bbE\insq{\vecA\bx + z_i|\tilde{\bx}= \bar{\bx}, \vecA_{i, j\in \cS_{\tilde{\bx}}} = \mathbf{a}_{i, j\in \cS_{\tilde{\bx}}},\bx}}&= \bbE\insq{\bbE\insq{\mathbf{a}_{i, S_{\bx}\cap S_{\tilde{\bx}}}|\tilde{\bx}= \bar{\bx}, \vecA_{i, j\in \cS_{\tilde{\bx}}} = \mathbf{a}_{i, j\in \cS_{\tilde{\bx}}},\bx}}\\
&\stackrel{(a)}{=} \frac{k-l}{k}\mathbf{a}_{i, S_{\tilde{\bx}}} 
\end{align*}where $(a)$ follows from \eqref{eq:1}.
\begin{align*}
&\inp{\bbE\insq{y_i|\tilde{\bx}= \bar{\bx}, \vecA_{i, j\in \cS_{\tilde{\bx}}} = \mathbf{a}_{i, j\in \cS_{\tilde{\bx}}}}}^2\\
&= \inp{\frac{k-l}{k}}^2\inp{\sum_{j \in \cS_{\tilde{\bx}}}a^2_{i,j} + 2\sum_{\stackrel{j,l\in \cS_{\tilde{\bx}}}{j\neq l}} a_{i,j}a_{i,l}}
\end{align*}
On the other hand,
\begin{align*}
 \bbE&\insq{\inp{y_i}^2|\tilde{\bx}= \bar{\bx}, \vecA_{\tilde{\bx}} = \mathbf{a}_{i, j\in \cS_{\tilde{\bx}}}}\\
 &= \bbE_{\bx}\insq{l+\sigma^2 + \inp{\mathbf{a}_{i, S_{\bx}\cap S_{\tilde{\bx}}}}^2}\\
 & \stackrel{(a)}{=} l + \sigma^2 + \frac{{k-1 \choose k-l-1}}{{k \choose k-l}}\sum_{j \in S_{\tilde{\bx}}}a^2_{i,j} + 2\frac{{k-2 \choose k-l-2}}{{k \choose k-l}}\sum_{\stackrel{j, l\in \cS_{\tilde{\bx}}}{j\neq l}}a_{i,j}a_{i,l}\\
 & = l + \sigma^2 + \frac{k-l}{k}\sum_{j \in S_{\tilde{\bx}}}a^2_{i,j} + 2\inp{\frac{k-l}{k}}\inp{\frac{k-l-1}{k-1}}\sum_{\stackrel{j, l\in \cS_{\tilde{\bx}}}{j\neq l}}a_{i,j}a_{i,l}
\end{align*} where $(a)$ follows from \eqref{eq:1} and \eqref{eq:2}.
Thus,
\begin{align*}
&\mathsf{Var}\inp{y_i-\bbE\insq{y_i|\tilde{\bx}= \bar{\bx}, \vecA_{i, j\in \cS_{\tilde{\bx}}}=\mathbf{a}_{i, j\in \cS_{\tilde{\bx}}}}\Big|\tilde{\bx}= \bar{\bx}, \vecA_{i, j\in \cS_{\tilde{\bx}}} = \mathbf{a}_{i, j\in \cS_{\tilde{\bx}}}}\\
& = l + \sigma^2 + \frac{k-l}{k}\sum_{j \in S_{\tilde{\bx}}}a^2_{i,j} + 2\inp{\frac{k-l}{k}}\inp{\frac{k-l-1}{k-1}}\sum_{\stackrel{j, l\in \cS_{\tilde{\bx}}}{j\neq l}}a_{i,j}a_{i,l}- \inp{\frac{k-l}{k}}^2\inp{\sum_{j \in \cS_{\tilde{\bx}}}a^2_{i,j} + 2\sum_{\stackrel{j,l\in \cS_{\tilde{\bx}}}{j\neq l}} a_{i,j}a_{i,l}}\\
&=  l + \sigma^2 + \inp{\frac{k-l}{k}}\inp{\frac{l}{k}}\sum_{j \in S_{\tilde{\bx}}}a^2_{i,j} -2\frac{k-l}{k}\frac{l}{k\inp{k-1}}\sum_{\stackrel{j,l\in \cS_{\tilde{\bx}}}{j\neq l}} a_{i,j}a_{i,l}
\end{align*}

Thus, 
\begin{align*}
h(y_i&|\vecA_{i, j\in \cS_{\tilde{\bx}}}, \tilde{\bx} =\tilde{\bx}) = \int p_{\vecA}(\ba)h(y_i|\tilde{\bx}= \bar{\bx}, \vecA_{i, j\in \cS_{\tilde{\bx}}}=\mathbf{a}_{i, j\in \cS_{\tilde{\bx}}})da\\
&\leq \int p_{\vecA}(\ba)\frac{1}{2}\log\inp{2\pi e\inp{l + \sigma^2 + \inp{\frac{k-l}{k}}\inp{\frac{l}{k}}\sum_{j \in S_{\tilde{\bx}}}a^2_{i,j} -2\frac{k-l}{k}\frac{l}{k\inp{k-1}}\sum_{\stackrel{j,l\in \cS_{\tilde{\bx}}}{j\neq l}} a_{i,j}a_{i,l}}}d{\ba}\\
&  \stackrel{(a)}{\leq} \frac{1}{2}\log\inp{2\pi e\inp{l + \sigma^2 + \inp{\int p_{\vecA}(\ba)\inp{\frac{k-l}{k}}\inp{\frac{l}{k}}\inp{\sum_{j \in S_{\tilde{\bx}}}a^2_{i,j} -2\frac{k-l}{k}\frac{l}{k\inp{k-1}}\sum_{\stackrel{j,l\in \cS_{\tilde{\bx}}}{j\neq l}} a_{i,j}a_{i,l}}d{\ba}}}}\\
&  \stackrel{(b)}{=} \frac{1}{2}\log\inp{2\pi e\inp{l + \sigma^2 + \inp{\frac{k-l}{k}}l}}\\
\end{align*}where $(a)$ follows from Jenson's inequality and $(b)$ follows by noting that $\bbE\insq{A_{i,j}^2} =1$  and $\bbE\insq{A_{i,j}A_{i,l}} =0$ for any $i$ and $j, l$, where $j\neq l$ and $\tilde{\bx}$ is $k$-sparse.

Thus, 
\begin{align*}
\sum_{i = 1}^mI(y_i;\bx|\vecA, y_{j\in [1:i-1]},\tilde{\bx})&\leq m\frac{1}{2}\log\inp{2\pi e\inp{l + \sigma^2 + \inp{\frac{k-l}{k}}l}} - \frac{1}{2}\log\inp{2\pi e \sigma^2}\\
& = \frac{m}{2}\log{\inp{1+ \frac{l}{\sigma^2}\inp{2-\frac{l}{k}}}}
\end{align*}
Using this and \eqref{eq:fano_spl}, we conclude that
\begin{align*}
m\geq \frac{nN(l) - 2\log{n}- h_2(\delta) - \delta k\log{n}}{\frac{1}{2}\log{\inp{1+ \frac{l}{\sigma^2}\inp{2-\frac{l}{k}}}}}.
\end{align*}

\end{proof}

\section{Comparison with \cite{vershyninPlan}}\label{sec:comparison_PV}
Algorithm~\ref{alg:1} is similar to the two step estimation procedure outlined in \cite{vershyninPlan} which was given to estimate the unknown signal within a two norm guarantee. 
Computing the vector $\mathbf{l} = \inp{l_1, \ldots, l_n}$ is the same as the first step of the procedure in [Section~1.2]\cite{vershyninPlan} where a linear estimator is computed. The second step of our algorithm (sorting and keeping the top-$k$ indices) can be thought of as a projection on a feasible set [Section~1.3]\cite{vershyninPlan}. However, this requires the estimation error to be small enough for the exact recovery of a binary vector.

The setup in \cite{vershyninPlan} is for the recovery of an unknown signal with small two-norm error, whereas our problem of exact recovery of a sparse binary vector is more suited for recovery under  infinity norm. This results in weak bounds ($m\approx O(k^2)$) when we specialize various results in \cite{vershyninPlan} to our case. We first note that we require $\bbE\norm{\frac{\hat{x}}{\norm{\hat{x}}}-\bar{x}}< \sqrt{\frac{{2}}{{k}}}$ for exact recovery. Otherwise, there exist two binary $k$-sparse vectors which have hamming distance at least two. 

We first consider the 1-bit compressed sensing result in Section 3.5 (page 13). Setting the LHS to $\sqrt{\frac{{2}}{{k}}}$, we get
\begin{align*}
\sqrt{\frac{{2}}{{k}}}\leq C\sqrt{\frac{k\log\inp{2n/k}}{m}}.
\end{align*}
This implies that $m\approx C_1 k^2\log\inp{2n/k}$ for some constant $C_1$.

Next, we consider [Theorem 9.1]\cite{vershyninPlan}. Note that for 1-bit compressed sensing $\eta^2 = 1$ and 
\begin{align*}
\mu &= \bbE\insq{s_1\ipr{a_1}{\bar{x}}}\\
& =\bbE\insq{s_1\ipr{a_1}{{x}}}\\
& \stackrel{(a)}{=} \frac{1}{\sqrt{k}}\sqrt{\frac{2}{\pi}}\times\frac{k}{\sqrt{\inp{k+\sigma^2}}}\\
&= \sqrt{\frac{2}{\pi}}\times\frac{\sqrt{k}}{\sqrt{\inp{k+\sigma^2}}}.
\end{align*} where $(a)$ follows from \eqref{eq:expt4}. 
Then, 
\begin{align*}
\norm{x-\mu\bar{x}} &= \norm{x-\sqrt{\frac{2}{\pi}}\times\frac{\sqrt{k}}{\sqrt{\inp{k+\sigma^2}}}\frac{x}{\sqrt{k}}}\\
& = \norm{x-\sqrt{\frac{2}{\pi}}\times\frac{x}{\sqrt{\inp{k+\sigma^2}}}}
\end{align*} We require $\norm{x-\mu\bar{x}}<\frac{2}{\sqrt{\pi\inp{k+\sigma^2}}}$ in order to exactly recover the unknown signal $x$.

We assume that $K$ is also a closed cone in $\bbR^n$. Then, by [Section~2.4]\cite{vershyninPlan}, $w_t(K) = tw_1(K) \leq t C\sqrt{k\log\inp{2n/k}}$ ([Section~2.4]\cite{vershyninPlan}). We choose $s = w_1(K)$. Substituting the bound for LHS and taking the limit $t\rightarrow 0$, we get
\begin{align*}
\frac{2}{\sqrt{\pi\inp{k+\sigma^2}}}\leq \frac{8 C \sqrt{k\log\inp{2n/k}}}{\sqrt{m}}.
\end{align*} Thus, $m\approx 4C(k+ \sigma^2)k\log\inp{2n/k}$.



\end{document}